\documentclass[
sigconf = true, 
review = false, 
screen = true, 
anonymous = false, 
nonacm = true
]{acmart}

\usepackage{amsmath,amsthm}
\usepackage{mathtools}
\usepackage{wasysym} 
\usepackage{tikz}
\usepackage{tkz-euclide}
\usepackage{xspace}
\usepackage{epsfig}
\usepackage{color}
\usepackage{url}
\usepackage{graphicx,dsfont}
\usepackage{tikz}
\usepackage{float}
\usepackage{longtable}
\usepackage{booktabs}
\usepackage{makecell}
\usepackage{multicol}
\usepackage{multirow}
\usepackage{colortbl}
\usepackage{tikz}
\usepackage{csquotes}
\usepackage{algpseudocode}
\usepackage[vlined,linesnumbered,ruled,titlenotnumbered]
{algorithm2e}

\definecolor{uniform1}{HTML}{1B9E77}
\definecolor{uniform2}{HTML}{D95F02}
\definecolor{uniform3}{HTML}{7570B3}
\definecolor{poisson}{HTML}{E7298A}

\newcommand{\ignore}[1]{}

\newcommand{\EA}{$(\mu+1)$-EA\xspace}


\begin{document}

\title{Evolutionary Diversity Optimization and the Minimum Spanning Tree Problem}


\author{Jakob Bossek}
\affiliation{%
  \institution{Statistics and Optimization\\Dept. of Information Systems\\University of M\"unster}
  \city{M\"unster, Germany}
  \country{}
}

\author{Frank Neumann}
\affiliation{%
  \institution{Optimisation and Logistics\\School of Computer Science\\The University of Adelaide}
  \city{Adelaide, Australia}
  \country{}
}

\begin{abstract}
In the area of evolutionary computation the calculation of diverse sets of high-quality solutions to a given optimization problem has gained momentum in recent years under the term evolutionary diversity optimization.
Theoretical insights into the working principles of baseline evolutionary algorithms for diversity optimization are still rare. In this paper we study the well-known Minimum Spanning Tree problem (MST) in the context of diversity optimization where population diversity is measured by the sum of pairwise edge overlaps. Theoretical results provide insights into the fitness landscape of the MST diversity optimization problem pointing out that even for a population of $\mu=2$ fitness plateaus (of constant length) can be reached, but nevertheless diverse sets can be calculated in polynomial time. We supplement our theoretical results with a series of experiments for the unconstrained and constraint case where all solutions need to fulfill a minimal quality threshold. Our results show that a simple $(\mu+1)$-EA can effectively compute a diversified population of spanning trees of high quality.
\end{abstract}

\maketitle


\section{Introduction}
\label{sec:sec1}

Evolutionary algorithms and other bio-inspired algorithms have successfully been applied to a wide range of challenging design and optimization problems.
Evolutionary algorithms use a population of search points in order to  solve a given problem. In the area of single-objective optimization the classical goal is to produce a single solution that maximizes or minimizes a given function $f$.
Diversity plays a crucial role in the design of evolutionary algorithms as it prevents the algorithm from getting stuck in a single local optima. It also enables the use of crossover which assumes that two good and different solutions can be combined to a new even better solution.

Evolutionary diversity optimization aims to produce for a given optimization problem a set of high quality and diverse solutions~\cite{DBLP:conf/gecco/UlrichT11,DBLP:journals/corr/abs-1802-05448,DBLP:conf/gecco/NeumannG0019}. A similar approach is taken in the context of quality diversity~\cite{lehman2011evolving,pugh2016extended} where high quality designs with different properties are sought.
Evolutionary diversity optimization has been applied in the context of creating diverse images with respect to various features~\cite{DBLP:conf/gecco/AlexanderKN17} as well as to the task of evolving instances for the Traveling Salesperson Problem (TSP) that show performance differences for algorithms solving the TSP~\cite{DBLP:conf/ppsn/GaoNN16,Bossek2019Evolving}. Such instances are important for automated algorithm selection and configuration~\cite{DBLP:journals/ec/KerschkeHNT19}.
Recently, evolutionary diversity optimization has also been applied to create a diverse set of high quality tours for the TSP~\cite{Do2020EvolvingDiverseTSPTours}.

In terms of theoretical foundations of evolutionary computing techniques, the area of runtime analysis has played a crucial role over the last 25 years. We refer to \cite{BookDoeNeu,DBLP:books/daglib/0025643,ncs/Jansen13} for comprehensive presentations on this research area.
So far, the effect of populations has been analyzed with respect to benefits or detrimental effects for solving single-objective optimization problems~\cite{DBLP:journals/ec/Storch08,DBLP:journals/algorithmica/DangJL17}. Furthermore, populations play a crucial role in providing benefits to multi-objective approaches. Recent studies in the context of Pareto optimization have shown that populations are highly beneficial for optimizing monotone submodular functions with different types of constraints~\cite{DBLP:journals/ec/FriedrichN15,DBLP:conf/nips/QianYZ15,DBLP:conf/ijcai/QianSYT17,DBLP:conf/aaai/RoostapourN0019,DBLP:conf/ppsn/NeumannN20}.
In the context of evolutionary diversity optimization, initial results have been obtained for the classical functions OneMax and LeadingOnes~\cite{DBLP:conf/gecco/GaoN14}. Furthermore, special instances of the vertex cover has been investigated and runtime results have been achieved~\cite{DBLP:conf/gecco/GaoPN15}. Still, time complexity results in the context of evolutionary diversity optimization are rare. This is due to the fact that the interactions of individuals in a population together with a measure of diversity is very hard to analyse.
However, theoretical foundations on well-studied combinatorial optimization problem are of utmost importance in order to get a deeper understanding of the fitness landscapes posed by these kind of problems and implied challenges for randomized search heuristics such as evolutionary algorithms.

\subsection{Our contribution}
We contribute to the theoretical analysis of evolutionary diversity optimization and contribute to the highly challenging task of understanding the interactions among individuals in a population when carrying out evolutionary diversity optimization.
In this paper, we consider evolutionary diversity optimization for the minimum spanning tree problem. This classical problem can be solved with polynomial-time algorithms, e.g. Kruskal~\cite{Kruskal56} or Prim~\cite{Prim57}. The MST problem has received significant attention in the area of runtime analysis over the years. Furthermore, different variants of the minimum spanning tree problem such as the multi-objective minimum spanning tree problem and degree or diameter constrained versions are $\mathcal{NP}$-hard.

In the evolutionary diversity optimization process we minimise pairwise edge overlap in the setting where $1\leq \mu \leq \lfloor\frac{n}{2}\rfloor$ which allows for a decomposition into a set of $\mu$ edge-disjoint -- and hence maximally diverse -- spanning trees.
We are interested in obtaining a diverse set of spanning trees in this context and start our investigations for the case where all edges have the same cost. This implies that we are only interested in a diverse population without imposing a quality criterion.
Our theoretical investigations point out structural properties when optimising the diversity of spanning trees using evolutionary algorithms and we give a first runtime analysis for the case of a population size of~$2$ (Theorem~\ref{thm:EA_mu_equal_two_runtime}).
Afterwards, we carry out experimental investigations for mutation-based evolutionary diversity optimization approaches to see whether simple operators are able to achieve maximal diversity when all edge weights are equal. We show that a
$(\mu+1)$-EA using single edge exchanges is efficient in finding a $\mu$-size population if no quality requirements are required. The process can be made more efficient if the number of edge exchanges is drawn from a Poisson distribution which enables several edge exchanges in one mutation step.
Furthermore, we observe that with $\mu \to \lfloor\frac{n}{2}\rfloor$ maximum overlap diversity goes hand in hand with reduced diversity with respect to interesting tree properties like the maximum degree or the diameter.
Having gained insights into the behavior for the unconstrained case, we turn to the constrained case where we require that each solution in the population can have weight at most $(1+\alpha)\cdot$OPT. Here OPT is the weight of a MST and $\alpha>0$ is a parameter that determines the maximum cost of a spanning tree to be deemed of high quality. We observe that the choice of $\alpha$ is crucial for the amount of diversity that can be achieved as small values of $\alpha$ determine that only a small number of (often very similar) spanning trees meet the quality criterion.

The paper is structured as follows. We introduce the minimum spanning tree problem and the evolutionary diversity optimization approach that we use in this article in Section~\ref{sec:sec2}. We investigate some structural properties for diversity optimization of spanning trees and present some runtime results subsequently in Section~\ref{sec:sec3}. Afterwards, we carry out experimental investigations in terms of the diversity optimization process for the unweighted case (Section~\ref{sec:sec4}) as well as evolutionary diversity optimization for the weighted case under quality restrictions (Section~\ref{sec:sec5}). We finish in Section~\ref{sec:sec6} with some concluding remarks and directions for future work.


\section{Problem Formulation}
\label{sec:sec2}

Let $G=(V,E)$ be a undirected, connected graph with node set $V$, edge set $E$ and cost function $c : E \to \mathbb{R^{+}}$ which assigns positive real-valued costs to each edge. In the following we denote by $n = |V|, m = |E|$ the size of the node and edge set respectively. Furthermore, we frequently use the notation $E(G)$ to refer to the edge set of graph $G$. We also frequently identify a graph simply by his edge set $E$ for convenience.
The Minimum Spanning Tree (MST) problem is a fundamental combinatorial optimization problem on graphs with countless applications~\cite{cormen01introduction}.
Each connected acyclic sub-graph $T = (V, E')$ with $E' \subset E$ is a spanning-tree (ST). A ST $T^{*}$ is a minimum spanning tree if its sum of weights is minimal across the set of all possible spanning trees $\mathcal{T}$, i.e.
\begin{align*}
    T^{*} = \text{arg\,min}_{T \in \mathcal{T}} \sum_{e \in E(T)} c(e).
\end{align*}
The MST problem is solvable in polynomial time, e.g. by Prim's algorithm~\cite{Prim57}.

\subsection{Evolutionary Diversity Optimization for the MST}

\begin{algorithm}[t]
Initialise the population $P$ with $\mu$ spanning trees such that $c(T) \leq (1+\alpha)\cdot \text{OPT}$ for all $T \in P$.\\
Choose $T \in P$ uniformly at random and produce an offspring $T'$ of $T$ by applying a single one-edge-exchange mutation.\\
If $c(T') \leq (1+ \alpha)\cdot \text{OPT}$, add $T'$ to $P$. \\
If $|P| = \mu+1$, remove exactly one individual $T$, where $T=\arg \min_{S \in P} D(P\setminus \{S\})$, from $P$.\\
Repeat steps 2 to 4 until a termination criterion is reached.\\
Return $P$
\caption{Diversity maximising ($\mu + 1$)-EA}
\label{alg:ea}
\end{algorithm}

In this paper we investigate evolutionary diversity optimization for the MST, i.e. the goal is -- given an input graph $G$ -- to find a $\mu$-size population $P$ of spanning trees which is maximal with respect to some diversity measure. In the context of \emph{minimum} spanning trees all solutions need to fulfill a minimum-quality criterion. To this end we adopt the notion of approximation quality: a solution $T$ is within the quality threshold if $c(T) \leq (1+\alpha)\cdot\text{OPT}$ where parameter $\alpha>0$ steers the quality requirement and OPT is the value of an MST calculated beforehand.

Our study is based on a simple \EA which has been studied extensively in the context of diversity optimization~\cite{DBLP:conf/gecco/GaoN14,DBLP:conf/gecco/AlexanderKN17,DBLP:conf/gecco/NeumannG0019}. The algorithmic steps are outlined in Algorithm~\ref{alg:ea}. The \EA requires a population $P$ of $\mu=|P|$ spanning trees which all meet a given quality criterion, i.e. $c(T) \leq (1+\alpha)\cdot\text{OPT}$ for some $\alpha>0$ and all $T \in P$. This may result from the application of a classic MST-algorithm such as Prim~\cite{Prim57}. Next, a solution is sampled from $P$ uniformly at random. This solution is subject to a One-Edge-Exchange (1-EX) mutation to produce a single offspring $T'$. 1-EX has been subject of various studies in the context of (multi-objective) minimum spanning trees~\cite{RKJ2006_BiasedMutationOperators,BNG2019}. It has the pleasing characteristic of maintaining the spanning tree property. One step includes adding a random edge to the solution at hand; this step closes exactly one cycle. Finally, a random edge is deleted from this cycle to re-establish the spanning tree property. This implies that mutants are guaranteed to be STs again.
The algorithm verifies the mutants quality and, if met, adds $T'$ to the population. Eventually a single solution is required to be dropped to maintain a population size of $\mu$ STs. Here, the algorithm removes the individual with the least contribution to the diversity measure. These steps are repeated until a a termination condition is met, e.g. a population with known maximum diversity is found (if known) or the function evaluation limit is reached.

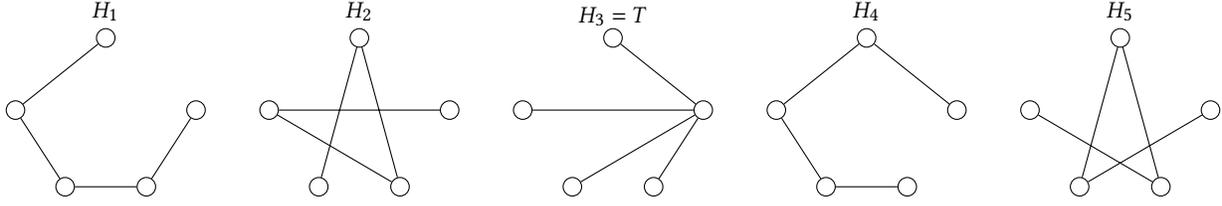
\begin{figure*}[ht]
\centering
\begin{tikzpicture}[scale=1.2] 
\begin{scope}[every node/.style={draw,circle,minimum size=7pt, inner sep=0pt}]
\node[label=above:{$H_1$}] (v1) at (0,0.8) {};
\node (v2) at (1,0) {};
\node (v3) at (0.45,-0.85) {};
\node (v4) at (-0.45,-0.85) {};
\node (v5) at (-1,0) {};
\end{scope}
\draw (v2) -- (v3) -- (v4) -- (v5) -- (v1);
\end{tikzpicture}
\hskip18pt
\begin{tikzpicture}[scale=1.2] 
\begin{scope}[every node/.style={draw,circle,minimum size=7pt, inner sep=0pt}]
\node[label=above:{$H_2$}] (v1) at (0,0.8) {};
\node (v2) at (1,0) {};
\node (v3) at (0.45,-0.85) {};
\node (v4) at (-0.45,-0.85) {};
\node (v5) at (-1,0) {};
\end{scope}
\draw (v2) -- (v5) -- (v3) -- (v1) -- (v4);
\end{tikzpicture}
\hskip18pt
\begin{tikzpicture}[scale=1.2] 
\begin{scope}[every node/.style={draw,circle,minimum size=7pt, inner sep=0pt}]
\node[label={[above,yshift=-0.32cm]:$H_3=T$}] (v1) at (0,0.8) {};
\node (v2) at (1,0) {};
\node (v3) at (0.45,-0.85) {};
\node (v4) at (-0.45,-0.85) {};
\node (v5) at (-1,0) {};
\end{scope}
\draw (v2) -- (v1);
\draw (v2) -- (v3);
\draw (v2) -- (v4);
\draw (v2) -- (v5);
\end{tikzpicture}
\hskip18pt
\begin{tikzpicture}[scale=1.2] 
\begin{scope}[every node/.style={draw,circle,minimum size=7pt, inner sep=0pt}]
\node[label=above:{$H_4$}] (v1) at (0,0.8) {};
\node (v2) at (1,0) {};
\node (v3) at (0.45,-0.85) {};
\node (v4) at (-0.45,-0.85) {};
\node (v5) at (-1,0) {};
\end{scope}
\draw (v2) -- (v1) -- (v5) -- (v4) -- (v3);
\end{tikzpicture}
\hskip18pt
\begin{tikzpicture}[scale=1.2] 
\begin{scope}[every node/.style={draw,circle,minimum size=7pt, inner sep=0pt}]
\node[label=above:{$H_5$}] (v1) at (0,0.8) {};
\node (v2) at (1,0) {};
\node (v3) at (0.45,-0.85) {};
\node (v4) at (-0.45,-0.85) {};
\node (v5) at (-1,0) {};
\end{scope}
\draw (v2) -- (v4) -- (v1) -- (v3) -- (v5);
\end{tikzpicture}
\caption{Illustration of set $H = \{H_1, \ldots, H_n\}$ for $n=5$ constructed in Theorem~\ref{thm:existence_population_max_min_difference}.}
\label{fig:proof_set_H}
\end{figure*}

\subsection{Diversity Measure}
Our diversity measure is based on the edge overlap of STs and has been recently used in a study on evolutionary diversity optimization for the Traveling Salesperson Problem~\cite{Do2020EvolvingDiverseTSPTours}. For two trees $T_1, T_2$ the \emph{(edge) overlap} is defined as the number of shared edges, i.e. $o(T_1, T_2) := |E(T_1) \cap E(T_2)|$. In our definition two trees are maximally diverse if they are \emph{edge-disjoint}, i.e. $o(T_1, T_2) = 0$ or equivalently -- as diversity is naturally to be maximised -- $(n-1) - o(T_1, T_2) = (n-1)$. We generalize this notion to a population $P = \{T_1, \ldots, T_{\mu}\}$ of $\mu$ STs by maximising the pairwise overlap
\begin{align}
\label{eq:population_overlap}
D_o(P) := \mu(\mu-1)(n-1) - \sum_{i=1}^{\mu}\sum_{j=1 \atop j \neq i}^{\mu} o(T_i, T_j).
\end{align}
Note that for fixed $n$ and $\mu$ the term $\mu(\mu-1)(n-1)$ is constant and maximising Eq.~\eqref{eq:population_overlap} is equivalent to minimising the pairwise overlap.

We next give a set of theoretical results pointing out that there exist populations which equalize the frequency of the edges. In particular, under the restriction $1 \leq \mu \leq \lfloor\frac{n}{2}\rfloor$, a population with maximum $D_o$-diversity, termed \emph{$D_o$-maximal} in the following, is achievable. Given a population $P$ we denote by $n(e, P) = |\{T \in P \mid e \in T\}|$ the number of STs the edge $e$ is part of. The goal is to equalize $n(e,P)$ for all $e \in E$.

\begin{theorem}
\label{thm:existence_population_max_min_difference}
For every complete graph $G=(V,E)$ with $n \geq 4$ nodes and every $\mu \geq 1$ there is a population $P$ of $\mu$ spanning trees with
\begin{align*}
\max_{e \in E}n(e, P) - \min_{e \in E}n(e,P) \in \{0, 1\}.
\end{align*}
\end{theorem}
\begin{proof}
We provide a constructive proof. The basic idea is to first generate a set $H$ of spanning trees and to fill $P$ with elements of $H$. We consider two cases: $n$ even and $n$ odd.
First let $n$ be even. Then there is a decomposition of $G$ into $h=\frac{n}{2}$ edge-disjoint $n$-vertex paths~\cite{Chartrand10986GraphsAndDigraphs}. Let $H$ be the set of those paths. Let $\mu=kh+r$ with $r \in [0, h)$. We put each $k$ copies of each $T \in H$ into $P$. Subsequently, we put an arbitrary subset $H'\subset H$ with $|H'|=r$ into $P$. Let $E(H')$ be the set of edges used by $H'$. Then we have $n(e, P)=k+1$ for all $e \in E(H')$ and $n(e,P)=k$ for all $e \notin E(H')$ yielding the claim. This completes the proof for even $n$.

Now let $n$ be odd. According to a well-known paper by Alspach et al.~\cite{Alspach1990} there is a decomposition of $G$ into $\lfloor\frac{n}{2}\rfloor = \frac{n-1}{2}$ edge-disjoint Hamiltonian cycles $C = \{C_1, \ldots, C_{(n-1)/2}\}$. Note that removing a single edge from a Hamiltonian cycle yields a Hamiltonian path which is a spanning tree. Since the cycles are edge-disjoint, so are the resulting spanning trees. Note further that each node in $G$ has degree~2 in each Hamiltonian cycle $C_i$. Let $v \in V$ be an arbitrary node and let $T$ be the star graph\footnote{A star graph is a spanning tree with a center node of degree $(n-1)$ and $(n-1)$ leaf nodes.} with center node $v$; $T$ is a spanning tree. We denote by $e_i^1$ and $e_i^{2}$ the two adjacent edges of $v$ for every cycle $C_i$. We now define a set $H = \{H_1, \ldots, H_n\}$ of $n$ spanning trees as follows:
\begin{itemize}
    \item $H_i = C_i \setminus \{e_i^1\}$ for $i = 1, \ldots, \frac{n-1}{2}$
    \item $H_{(n-1)/2+1} = T$ and
    \item $H_{(n-1)/2+1+i} = C_i \setminus \{e_i^2\}$ for $i=1, \ldots, \frac{n-1}{2}$.
\end{itemize}
$T$ is constructed in a way such that it shares each one edge with every spanning tree in $H\setminus \{T\}$. Likewise, the spanning trees $H_i$ and $H_{(n-1)/2+1+i}, i=1, \ldots, \frac{n-1}{2}$ share each $n-2$ edges. Note further that $H_1, \ldots, H_{(n-1)/2}$ are edge-disjoint and so are $H_{(n-1)/2+2}, \ldots, H_n$ (see Figure~\ref{fig:proof_set_H} for an illustration). As a consequence $n(e,H)=2$ for all $e \in E$. It follows by construction of $H$ that for each $1 \leq \mu \leq n$ the set $P_{\mu}=\{H_1, \ldots, H_{\mu}\} \subset H$ satisfies
\begin{align*}
    \alpha(P_{\mu}) := \max_{e \in E} n(e,P_{\mu}) - \min_{e \in E} n(e, P_{\mu}) \in \{0, 1\}.
\end{align*}
with $\alpha(P_n)=0$.

Now consider $\mu \geq 1$. We construct a $\mu$-size population $P$ by setting $P_i = H_{((i-1) \mod n)+1}$for $i = 1, \ldots, \mu$. The theorem follows from iterated application of the claim for the case $1 \leq \mu \leq n$.
\end{proof}

Following the proof of Theorem~\ref{thm:existence_population_max_min_difference} we can derive.

\begin{corollary}
\label{cor:existence_population_n_diff_zero}
For every complete graph $G=(V,E)$ with $|V|=n$ and every $k \geq 1$ there is a population of size $\mu=kn$ such that each edge is used exactly $2k$ times.
\end{corollary}

Moreover, for the case of $1 \leq \mu \leq \lfloor\frac{n}{2}\rfloor$ it follows that a set of $\mu$ edge-disjoint $STs$ is possible. From now on this will be the case of interest unless told otherwise.

\begin{corollary}
\label{cor:existence_population_edge_disjoint_trees}
For every complete graph $G=(V,E)$ with $|V|=n$ and every $1 \leq \mu \leq \lfloor \frac{n}{2} \rfloor$ there is a population $P$ of $\mu$ edge-disjoint spanning trees with maximum diversity $D_o(P)=\mu(\mu-1)(n-1)$.
\end{corollary}

\section{Structural Properties And Runtime Results}
\label{sec:sec3}

We consider the easiest non-trivial case: $\mu=2$. We study the performance of $(2+1)$-EA (see Algorithm~\ref{alg:ea}) on complete graphs with $n \geq 4$ nodes where all edges have the same cost. We are interested in the \emph{expected number of function evaluations} until a $D_o$-maximal population is found for the first time; the most common measure in the the time complexity analysis of randomized search heuristics
~\cite{BookDoeNeu}. The fitness function for the EA is the edge overlap which we aim to minimise (cf.~Eq~\ref{eq:population_overlap}); recall that this is equivalent to maximising the overlap diversity.

Clearly, given two STs $T_1, T_2$ the fitness can be improved by replacing an overlap edge, i.e. and edge $e \in E(T_1) \cap E(T_2)$ with a \emph{free} edge which is not used by neither $T_1$ and $T_2$. It turns out that even for $\mu=2$ a single one-edge exchange mutation may not be sufficient to reduce the overlap. Consider the case where $T_1$ and $T_2$ are star graphs and there is just a single overlap edge $\{u,v\}$ which links the center nodes $u$ of $T_1$ and $v$ of $T_2$ (see Fig.~\ref{fig:fitness_plateau}). In this case any alternative edge for $\{u,v\}$ in $T_1$ and $T_2$ is also in $T_2$ and $T_1$. Hence, in this special constellation a fitness plateau is reached which cannot be surpassed by a single 1-EX move. However, two consecutive 1-EX moves are sufficient to surpass this fitness plateau: w.l.o.g. replace $\{u,v\}$ in $T_1$ with one of the other suitable edges. Now $T_1$ is no star graph anymore and an improving step for $T_2$ exists.

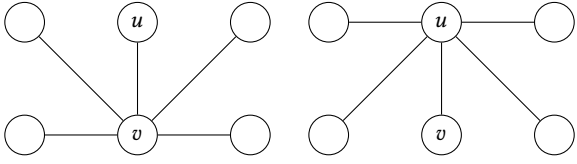
\begin{figure}[t]
\centering
\begin{tikzpicture}[xscale=1.5, yscale=1.5]
\begin{scope}[every node/.style={circle, draw, fill=white, minimum size=15pt}]
\node (v0) at (0,0) {};
\node (v1) at (1,0) {$v$};
\node (v2) at (2,0) {};
\node (v3) at (0,1) {};
\node (v4) at (1,1) {$u$};
\node (v5) at (2,1) {};

\foreach \i in {0, 2, 3, 4, 5}
{
    \draw (v1) edge[solid] (v\i);
}
\end{scope}
\end{tikzpicture}
\hskip12pt
\begin{tikzpicture}[xscale=1.5, yscale=1.5]
\begin{scope}[every node/.style={circle, draw, fill=white, minimum size=15pt}]
\node (v0) at (0,0) {};
\node (v1) at (1,0) {$v$};
\node (v2) at (2,0) {};
\node (v3) at (0,1) {};
\node (v4) at (1,1) {$u$};
\node (v5) at (2,1) {};

\foreach \i in {0, 1, 2, 3, 5}
{
    \draw (v4) edge (v\i);
}
\end{scope}
\end{tikzpicture}
\caption{Two spanning trees of star graph type with center nodes $v$ for $T_1$ (left) and $u$ for $T_2$ (right).}
\label{fig:fitness_plateau}
\end{figure}

In the following we show that in case of overlap $o(T_1, T_2) \geq 2$ there always exists a 1-EX step that strictly improves the diversity.

\begin{lemma}
\label{lem:free_edges_lemma}
On every complete graph with $n$ nodes and population $P = \{T_1, T_2\}$, if $o(T_1, T_2) \geq 2$ then each edge $e \in E(T_1) \cap E(T_2)$ can be replaced with $k-1$ other edges in either tree, such that the overlap reduces by one.
\end{lemma}

\begin{proof}
Let $k = o(T_1, T_2) \in \{2, \ldots, n-1\}$. Choose an arbitrary overlap edge $e \in (E(T_1) \cap E(T_2))$. W.~l.~o.~g. set $T_1' = T_1 \setminus \{e\}$. Since $T_1$ is a spanning tree, removing $e$ destroys connectivity and decomposes $T_1'$ into a cut, i.e. two disjoint node sets $C_1$ and $C_2$. It should be noted that 1-EX works differently: first insert an edge and remove an edge from the unique introduced cycle. However, for the sake of intuition we might think the other way around. In order to re-establish the spanning tree property in $T'$ a cut-edge (an edge that has one end node in $C_1$ and the other in $C_2$) needs to be added. There are $|C_1|\cdot|C_2|$ such edges since $G$ is complete. Now $|C_1|\cdot|C_2| \geq (n-1)$. This happens if one node set has a single isolated node.
Since $e$ is a cut-edge there are $(n-2)$ alternative cut-edges left. The overlap is $k$. In consequence at most $(n-1)-k$ of those $(n-2)$ alternatives are used by $T_2$. Hence, the number of non-used cut edges is at least $(n-2) - ((n - 1) - k) = k - 1 \geq 1$ since $k \geq 2$. Hence, we can always include one of those edges into $T_1$ and remove $e$ which leads to a reduction of the overlap by one.
\end{proof}

\begin{theorem}
\label{thm:EA_mu_equal_two_runtime}
On every complete graph with $n$ nodes $(2+1)$-EA with one-edge-exchange mutation, starting with two clones of an arbitrary spanning tree, needs expected time $O(mn\log n)$ to find a $D_o$-maximal population.
\end{theorem}

\begin{proof}
Let $T_1$ and $T_2$ be the two spanning trees. Let $k = o(T_1, T_2) \in \{0, \ldots, n-1\}$ with $k = n-1$ after initialization. We define natural disjoint fitness-levels~\cite{BookDoeNeu}
\begin{align*}
    A_k = \{P = \{T_1, T_2\} \,|\, o(T_1, T_2) = k\},
\end{align*}
i.~e., the $k$-th fitness level contains all populations with overlap $k$ and $A_0$ contains all populations of maximum $D_o$-diversity. In order to leave $A_k$ and transition into $A_{k-1}$ or lower (note that we are minimizing overlap here), we need to reduce the overlap by at least one. For $k \geq 2$ by Lemma~\ref{lem:free_edges_lemma} we know that we can include at least $(k-1)$ different edges into either $T_1$ or $T_2$ to produce a cycle with a cut edge on it. The probability for this event is at least $(k-1)/m$. The resulting cycle has length at most $n$. Thus, the probability to select the overlap edge is at least $1/n$.
By folklore waiting time arguments the expected waiting time for such an improving step is at most $(mn)/(k-1)$. Summing up we obtain the following upper bound until we reach $A_1$:
\begin{align*}
    \sum_{k=2}^{n-1} \frac{mn}{k-1} = mn \sum_{k=1}^{n-2} \frac{1}{k} = mn \cdot H_{n-2} = O(mn\log n).
\end{align*}
Here, $H_{n}=\sum_{i=1}^{n} (1/i) = \log(n) + \Theta(1)$ is the Harmonic sum.
The last step is to reach $A_0$ from $A_1$. This may be possible in a single step, but there might be a problem if $T_1$ and $T_2$ are star graphs where the overlap edge links the center nodes. This situation can be resolved by two consecutive mutations with an escape-state when one of the trees is no star graph anymore. In this situation a single one-edge exchange in the other tree can get rid of the last overlap. There is a fair random walk on the states of the fitness plateau which ends in the escape state in expected time $O(mn)$ as each step happens with probability at least $1/(mn)$. Eventually, eliminating the last overlap requires another $O(mn)$ steps in expectation since the overlap edge and an adequate replacement edge have to be selected; again, both events happen with probability at least $1/(mn)$. In sum the expected time to reach $A_0$ remains $O(mn \log n)$.
\end{proof}

\begin{table*}[htbp]
\centering
\caption{\label{tab:uniform_weights}Comparison in terms of diversity (\textbf{$D_o$}), the mean number of iterations (\textbf{mean}) until the algorithm terminated, the standard deviation of iterations (\textbf{std}) and results of Wilcoxon-Mann-Whitney tests (\textbf{stat}). The data is split by the sampling strategy (Uniform or Poisson). Here, $X^{+}$ means that the number of iterations is significantly better, i.e. lower, for the respective strategy. Lowest mean values per row are \colorbox{gray!20}{\textbf{highlighted}}.}
\renewcommand{\tabcolsep}{3pt}
\renewcommand{\arraystretch}{1.6}
\centering
\begin{footnotesize}
\begin{tabular}[t]{rrrrrrrrrrrrrrrrrr}
\toprule
\multicolumn{1}{c}{\textbf{ }} & \multicolumn{1}{c}{\textbf{ }} & \multicolumn{4}{c}{\textcolor{uniform1}{\textbf{Uniform[1] (1)}}} & \multicolumn{4}{c}{\textcolor{uniform2}{\textbf{Uniform[2] (2)}}} & \multicolumn{4}{c}{\textcolor{uniform3}{\textbf{Uniform[3] (3)}}} & \multicolumn{4}{c}{\textcolor{poisson}{\textbf{Poisson (4)}}} \\
\cmidrule(l{3pt}r{3pt}){3-6} \cmidrule(l{3pt}r{3pt}){7-10} \cmidrule(l{3pt}r{3pt}){11-14} \cmidrule(l{3pt}r{3pt}){15-18}
$n$ & $\mu$ & $D_o$ & \textbf{mean} & \textbf{std} & \textbf{stat} & $D_o$ & \textbf{mean} & \textbf{std} & \textbf{stat} & $D_o$ & \textbf{mean} & \textbf{std} & \textbf{stat} & $D_o$ & \textbf{mean} & \textbf{std} & \textbf{stat}\\
\midrule
 & 2 & 100.00 & 488.70 & 164.8 &  & 100.00 & 333.97 & 144.0 & \textcolor{uniform1}{$\text{1}^{+}$} & 100.00 & 243.93 & 104.5 & \textcolor{uniform1}{$\text{1}^{+}$}, \textcolor{uniform2}{$\text{2}^{+}$} & 100.00 & \cellcolor{gray!20}{\textbf{198.67}} & 85.2 & \textcolor{uniform1}{$\text{1}^{+}$}, \textcolor{uniform2}{$\text{2}^{+}$}, \textcolor{uniform3}{$\text{3}^{+}$}\\

 & 10 & 100.00 & 7153.43 & 2389.4 &  & 100.00 & 5289.17 & 1469.0 & \textcolor{uniform1}{$\text{1}^{+}$} & 100.00 & 4660.60 & 1514.0 & \textcolor{uniform1}{$\text{1}^{+}$}, \textcolor{uniform2}{$\text{2}^{+}$} & 100.00 & \cellcolor{gray!20}{\textbf{3741.43}} & 1185.4 & \textcolor{uniform1}{$\text{1}^{+}$}, \textcolor{uniform2}{$\text{2}^{+}$}, \textcolor{uniform3}{$\text{3}^{+}$}\\

\multirow{-3}{*}{\raggedleft\arraybackslash 50} & 25 & 99.80 & 62500.00 & 0.0 &  & 99.79 & 62500.00 & 0.0 &  & 99.77 & 62500.00 & 0.0 &  & 99.74 & 62500.00 & 0.0 & \\
\cmidrule{1-18}
 & 2 & 100.00 & 1624.93 & 510.2 &  & 100.00 & 1072.43 & 429.8 & \textcolor{uniform1}{$\text{1}^{+}$} & 100.00 & 642.27 & 283.8 & \textcolor{uniform1}{$\text{1}^{+}$}, \textcolor{uniform2}{$\text{2}^{+}$} & 100.00 & \cellcolor{gray!20}{\textbf{578.70}} & 255.8 & \textcolor{uniform1}{$\text{1}^{+}$}, \textcolor{uniform2}{$\text{2}^{+}$}\\

 & 10 & 100.00 & 16836.80 & 3096.7 &  & 100.00 & 11373.37 & 3580.1 & \textcolor{uniform1}{$\text{1}^{+}$} & 100.00 & 8120.70 & 2252.3 & \textcolor{uniform1}{$\text{1}^{+}$}, \textcolor{uniform2}{$\text{2}^{+}$} & 100.00 & \cellcolor{gray!20}{\textbf{7603.67}} & 2459.1 & \textcolor{uniform1}{$\text{1}^{+}$}, \textcolor{uniform2}{$\text{2}^{+}$}\\

 & 25 & 100.00 & 73498.87 & 18139.3 &  & 100.00 & 59131.93 & 15166.8 & \textcolor{uniform1}{$\text{1}^{+}$} & 100.00 & 58792.83 & 19515.9 & \textcolor{uniform1}{$\text{1}^{+}$} & 100.00 & \cellcolor{gray!20}{\textbf{52915.17}} & 18007.0 & \textcolor{uniform1}{$\text{1}^{+}$}, \textcolor{uniform2}{$\text{2}^{+}$}\\

\multirow{-4}{*}{\raggedleft\arraybackslash 100} & 50 & 99.93 & 500000.00 & 0.0 &  & 99.92 & 500000.00 & 0.0 &  & 99.91 & 500000.00 & 0.0 &  & 99.89 & 500000.00 & 0.0 & \\
\cmidrule{1-18}
 & 2 & 100.00 & 4366.17 & 1200.9 &  & 100.00 & 2757.70 & 999.9 & \textcolor{uniform1}{$\text{1}^{+}$} & 100.00 & 1981.87 & 809.1 & \textcolor{uniform1}{$\text{1}^{+}$}, \textcolor{uniform2}{$\text{2}^{+}$} & 100.00 & \cellcolor{gray!20}{\textbf{1523.53}} & 645.6 & \textcolor{uniform1}{$\text{1}^{+}$}, \textcolor{uniform2}{$\text{2}^{+}$}, \textcolor{uniform3}{$\text{3}^{+}$}\\

 & 10 & 100.00 & 45564.87 & 11035.6 &  & 100.00 & 30632.00 & 5848.0 & \textcolor{uniform1}{$\text{1}^{+}$} & 100.00 & 19364.20 & 3210.8 & \textcolor{uniform1}{$\text{1}^{+}$}, \textcolor{uniform2}{$\text{2}^{+}$} & 100.00 & \cellcolor{gray!20}{\textbf{16428.40}} & 3630.1 & \textcolor{uniform1}{$\text{1}^{+}$}, \textcolor{uniform2}{$\text{2}^{+}$}, \textcolor{uniform3}{$\text{3}^{+}$}\\

 & 25 & 100.00 & 153040.03 & 30063.4 &  & 100.00 & 114716.33 & 24364.6 & \textcolor{uniform1}{$\text{1}^{+}$} & 100.00 & 86105.70 & 25137.5 & \textcolor{uniform1}{$\text{1}^{+}$}, \textcolor{uniform2}{$\text{2}^{+}$} & 100.00 & \cellcolor{gray!20}{\textbf{74420.97}} & 19746.9 & \textcolor{uniform1}{$\text{1}^{+}$}, \textcolor{uniform2}{$\text{2}^{+}$}, \textcolor{uniform3}{$\text{3}^{+}$}\\

 & 50 & 100.00 & 448873.77 & 56975.2 &  & 100.00 & 422085.90 & 97636.7 & \textcolor{uniform1}{$\text{1}^{+}$} & 100.00 & 396190.57 & 60480.1 & \textcolor{uniform1}{$\text{1}^{+}$} & 100.00 & \cellcolor{gray!20}{\textbf{360588.67}} & 92265.0 & \textcolor{uniform1}{$\text{1}^{+}$}, \textcolor{uniform2}{$\text{2}^{+}$}, \textcolor{uniform3}{$\text{3}^{+}$}\\

\multirow{-5}{*}{\raggedleft\arraybackslash 200} & 100 & 99.98 & 4000000.00 & 0.0 &  & 99.97 & 4000000.00 & 0.0 &  & 99.96 & 4000000.00 & 0.0 &  & 99.95 & 4000000.00 & 0.0 & \\
\cmidrule{1-18}
 & 2 & 100.00 & 15081.70 & 4080.7 &  & 100.00 & 9765.10 & 3337.4 & \textcolor{uniform1}{$\text{1}^{+}$} & 100.00 & 5980.50 & 1873.9 & \textcolor{uniform1}{$\text{1}^{+}$}, \textcolor{uniform2}{$\text{2}^{+}$} & 100.00 & \cellcolor{gray!20}{\textbf{5139.60}} & 1834.1 & \textcolor{uniform1}{$\text{1}^{+}$}, \textcolor{uniform2}{$\text{2}^{+}$}, \textcolor{uniform3}{$\text{3}^{+}$}\\

 & 10 & 100.00 & 126942.27 & 29267.6 &  & 100.00 & 80306.63 & 17046.5 & \textcolor{uniform1}{$\text{1}^{+}$} & 100.00 & 56559.83 & 10089.8 & \textcolor{uniform1}{$\text{1}^{+}$}, \textcolor{uniform2}{$\text{2}^{+}$} & 100.00 & \cellcolor{gray!20}{\textbf{42918.23}} & 8707.2 & \textcolor{uniform1}{$\text{1}^{+}$}, \textcolor{uniform2}{$\text{2}^{+}$}, \textcolor{uniform3}{$\text{3}^{+}$}\\

 & 25 & 100.00 & 361950.37 & 43469.4 &  & 100.00 & 261919.33 & 46744.9 & \textcolor{uniform1}{$\text{1}^{+}$} & 100.00 & 189700.70 & 31737.1 & \textcolor{uniform1}{$\text{1}^{+}$}, \textcolor{uniform2}{$\text{2}^{+}$} & 100.00 & \cellcolor{gray!20}{\textbf{152788.93}} & 27460.6 & \textcolor{uniform1}{$\text{1}^{+}$}, \textcolor{uniform2}{$\text{2}^{+}$}, \textcolor{uniform3}{$\text{3}^{+}$}\\

 & 50 & 100.00 & 1017406.93 & 216403.9 &  & 100.00 & 731349.90 & 96224.3 & \textcolor{uniform1}{$\text{1}^{+}$} & 100.00 & 588732.20 & 83189.7 & \textcolor{uniform1}{$\text{1}^{+}$}, \textcolor{uniform2}{$\text{2}^{+}$} & 100.00 & \cellcolor{gray!20}{\textbf{543359.37}} & 112333.1 & \textcolor{uniform1}{$\text{1}^{+}$}, \textcolor{uniform2}{$\text{2}^{+}$}, \textcolor{uniform3}{$\text{3}^{+}$}\\

\multirow{-5}{*}{\raggedleft\arraybackslash 400} & 100 & 100.00 & 3344395.80 & 377191.7 &  & 100.00 & 3019514.87 & 550293.3 & \textcolor{uniform1}{$\text{1}^{+}$} & 100.00 & 2827195.20 & 454872.3 & \textcolor{uniform1}{$\text{1}^{+}$} & 100.00 & \cellcolor{gray!20}{\textbf{2541803.50}} & 495036.2 & \textcolor{uniform1}{$\text{1}^{+}$}, \textcolor{uniform2}{$\text{2}^{+}$}, \textcolor{uniform3}{$\text{3}^{+}$}\\
\bottomrule
\end{tabular}
\end{footnotesize}
\end{table*}

Theorem~\ref{thm:EA_mu_equal_two_runtime} states that \EA is capable of finding a population of maximum diversity in polynomial time for $\mu=2$; a satisfying result. However, the probabilities derived in Lemma~\ref{lem:free_edges_lemma} are rather pessimistic. The number of cut-edges which link the two connected components $C_1$ and $C_2$ when an edge is dropped from one tree is at least $(n-1)$. However, if $|C_1|=|C_2|=(n/2)$ ($n$ even), there are $n^2/4=\Theta(n^2)$ many cut edges. In consequence the probability to reduce the overlap is much higher. This could indicate that on average the running time is lower. We now turn the focus to experimental investigations with a broader range of $\mu$-values in the following sections.

\section{Experimental Study}
\label{sec:sec4}

We presented theoretical insights into the expected runtime of $(\mu+1)$-EA for the case $\mu=2$. In this section we perform a series of experiments for general $2 \leq \mu \leq \lfloor\frac{n}{2}\rfloor$. We start with the unconstrained case, i.e., spanning trees are not required to meet a minimum quality. After that we consider the constrained case with different quality thresholds.

\subsection{Reproducibility}

All experiments have been conducted on a HPC cluster with Intel Xeon Gold 6140 18C \@2.30GHz (Skylake) CPUs and 16GB RAM. Algorithm~\ref{alg:ea} was implemented in the statistical programming language R~\cite{Rlang} in version 3.6.0 for rapid prototyping interfacing C++ for performance. Parallel job distribution/management was accomplished with the R package \texttt{batchtools}~\cite{Rbatchtools}. We also relied on the random number generator adopted by \texttt{batchtools}.
Code and data are available in a public GitHub repository.\footnote{Code and data: \url{https://github.com/jakobbossek/GECCO2021-mst-diversity}.}


\subsection{Unconstrained Diversity Optimization}

In the unconstrained setting we consider complete graphs $G=(V,E)$ with $n \in \{50, 100, 200, 400\}$, $m = n(n-1)/2$ and $\mu$ in the range between $2$ and $\lfloor\frac{n}{2}\rfloor$. As our focus in on the algorithms ability to cope with the diversity part the considered setting adopts $c(e)=1 \,\forall e \in E$. This way the check in line~4 of Algorithm~\ref{alg:ea} is always true.
Note that for this setup the statement from Theorem~\ref{cor:existence_population_edge_disjoint_trees} holds.

We are interested in the number of function evaluations required by \EA to reach a population of maximum diversity. Therefore, we adopt two stopping criteria: (1) stop when a maximum diversity population is reached or (2) stop if the budget of $\mu n^2$ function evaluations (FE) is depleted. In addition, we study the effect of a single 1-EX, i.e. performing a single one-edge-exchange per mutation as stated in Algorithm~\ref{alg:ea} in comparison to possibly multiple subsequent 1-EX moves. Here, we consider the case where the number of swaps is sampled from $\{1,\ldots,l\}$ for $l=1,2,3$ uniformly at random (this is denoted by Uniform[$l$] in the following and includes the case of a single 1-EX move).
In addition we consider the number of 1-EX moves being sampled from a Poisson distribution with rate $\lambda=1$. Here, there is a strictly positive probability to exchange multiple edges which may be advantageous and necessary to leave local optima. For each $(n, \mu, \text{sampling})$-combination we run \EA 30 times for statistical soundness.

\begin{figure}[t]
    \centering
    \includegraphics[width=\columnwidth]{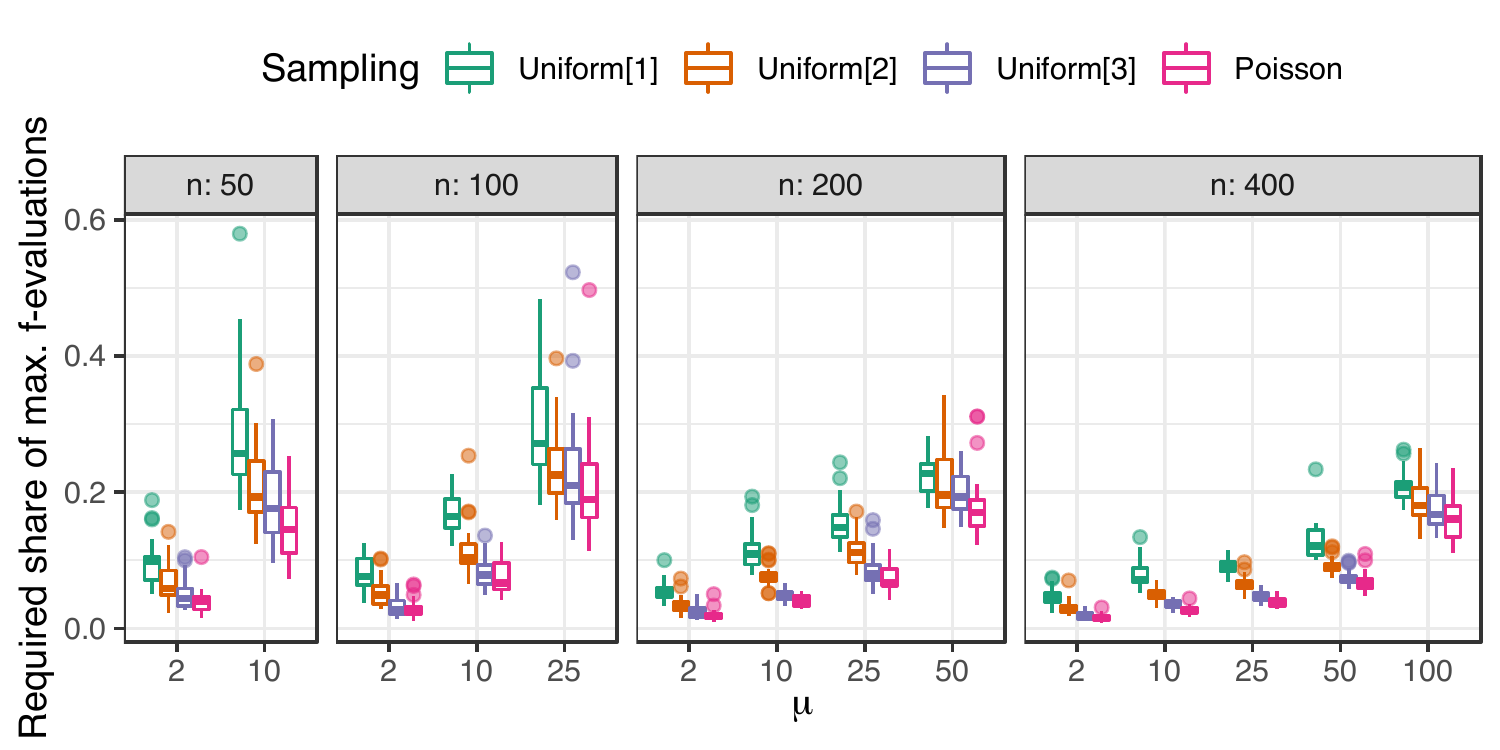}
    \caption{Distribution of the share of maximum function evaluations ($\mu n^2$) in the unconstrained case split by instance size $n$ and population size $\mu$.}
    \label{fig:boxplots_fevals}
\end{figure}
Table~\ref{tab:uniform_weights} shows a comparison in terms of population diversity $D_o$ in percent as well as summary statistics for the number of iterations required. The results indicate that \EA manages to find a $\mu$-size population of maximum diversity for almost all considered scenarios within the limit of $\mu n^2$ function evaluations. Solely the case where $\mu = \lfloor\frac{n}{2}\rfloor$ seems tricky. Here, \EA apparently gets trapped in local optima and is unable to escape (within the given FE-limit): all 30 runs consistently hit the imposed FE-limit.
In addition we observe a near-constant advantage for multiple 1-EX moves per mutation. For relative low $\mu$ \EA with Poisson sampling is up to three times as fast as its competitor and also shows lower standard deviation. The results are statistically significant as indicated by the $1^{+}$ values in the \textbf{stat} column; a plus symbol shows that the zero hypothesis of lower number of FEs for the competitor algorithm was rejected by the Wilcoxon-Mann-Whitney test at a significance level of $5\%$. This effect becomes weaker when $\mu$ approaches $\lfloor\frac{n}{2}\rfloor$. In this region it apparently becomes more beneficial -- in the end of the optimization process -- to produce mutants by single 1-EX moves.
\begin{figure*}[t]
    \centering
    \includegraphics[width=\textwidth]{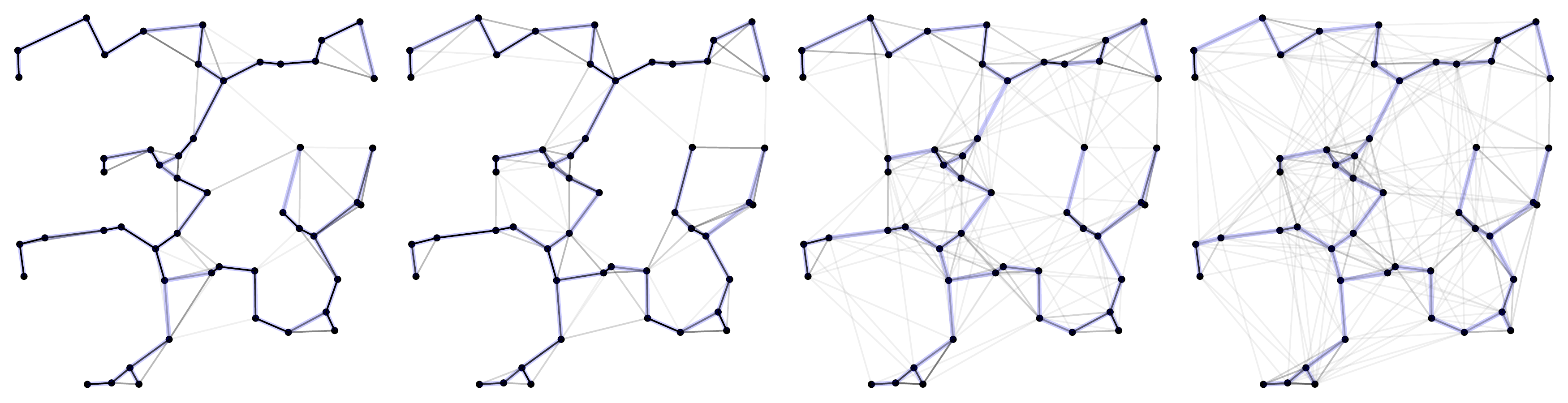}
    \caption{Visualised edge usage for an Euclidean instance with $n=50$ nodes, $\mu=10$ and $\alpha\in \{0.05, 0.10, 0.50, 1.00\}$ (from left to right). In the plots all $\mu$ spanning trees are superimposed. The darker an edge, the higher the overlap in the population. Omitted edges are not part of any solution. MST-edges are shown in blue.}
    \label{fig:superimposed_edges}
\end{figure*}
Figure~\ref{fig:boxplots_fevals} gives a visual impression of the number of function evaluations required as a share of the maximum FE-limit $\mu n^2$. The box-plots confirm our observations from Table~\ref{tab:unconstrained_diversity_measures}.
In conclusion, this series of experiments verifies empirically that \EA is able to efficiently evolve sets of edge-disjoint STs.

The focus of this paper is on the overlap-based diversity measure. However, in real-world applications, e.g. telecommunication networks, diversity with respect to additional structural features is desirable. To reduce the workload of single nodes in said networks, constraints on the maximum degree of the MST pose real-world requirements. In this context also MSTs with many leafs are advantageous. Furthermore, low-diameter\footnote{The diameter of a graph is the length of its longest simple path.} MSTs enable fast communication since the number of hops is kept relatively low~\cite{GM2003NetworkFlowModels}.
\begin{table}[tb]
\caption{\label{tab:unconstrained_diversity_measures}Diversity (in percent of the maximally achievable value) with respect to overlap diversity $D_o$, maximum degree (\textbf{Max. degree}), number of leafs (\textbf{Leaf}) and the diameter (\textbf{Diameter}) in the unconstrained setting.}
\renewcommand{\tabcolsep}{4.7pt}
\renewcommand{\arraystretch}{1.4}
\centering
\begin{footnotesize}
\centering
\begin{tabular}[t]{rrrrrrrrrr}
\toprule
\multicolumn{1}{c}{\textbf{ }} & \multicolumn{1}{c}{\textbf{ }} & \multicolumn{2}{c}{\textbf{$D_o$}} & \multicolumn{2}{c}{\textbf{Max. degree}} & \multicolumn{2}{c}{\textbf{Leaf}} & \multicolumn{2}{c}{\textbf{Diameter}} \\
\cmidrule(l{3pt}r{3pt}){3-4} \cmidrule(l{3pt}r{3pt}){5-6} \cmidrule(l{3pt}r{3pt}){7-8} \cmidrule(l{3pt}r{3pt}){9-10}
$n$ & $\mu$ & \textbf{mean} & \textbf{std} & \textbf{mean} & \textbf{std} & \textbf{mean} & \textbf{std} & \textbf{mean} & \textbf{std}\\
\midrule
 & 2 & 100.00 & 0.00 & 83.33 & 23.97 & 88.33 & 21.51 & 95.00 & 15.26\\

 & 10 & 100.00 & 0.00 & 30.67 & 6.40 & 61.67 & 11.17 & 65.67 & 11.65\\

\multirow{-3}{*}{\raggedleft\arraybackslash 50} & 25 & 99.80 & 0.02 & 14.67 & 2.43 & 33.33 & 4.25 & 41.07 & 5.03\\
\cmidrule{1-10}
 & 2 & 100.00 & 0.00 & 85.00 & 23.30 & 96.67 & 12.69 & 95.00 & 15.26\\

 & 10 & 100.00 & 0.00 & 30.67 & 7.40 & 66.33 & 9.64 & 76.33 & 9.99\\

 & 25 & 100.00 & 0.00 & 16.27 & 3.14 & 43.07 & 5.82 & 50.40 & 4.99\\

\multirow{-4}{*}{\raggedleft\arraybackslash 100} & 50 & 99.93 & 0.00 & 8.87 & 1.14 & 26.07 & 3.34 & 32.07 & 3.46\\
\cmidrule{1-10}
 & 2 & 100.00 & 0.00 & 75.00 & 25.43 & 95.00 & 15.26 & 96.67 & 12.69\\

 & 10 & 100.00 & 0.00 & 35.00 & 5.72 & 78.33 & 8.34 & 81.33 & 7.76\\

 & 25 & 100.00 & 0.00 & 15.20 & 2.44 & 53.73 & 5.91 & 60.13 & 6.52\\

 & 50 & 100.00 & 0.00 & 8.33 & 1.06 & 33.00 & 2.45 & 40.53 & 3.52\\

\multirow{-5}{*}{\raggedleft\arraybackslash 200} & 100 & 99.98 & 0.00 & 4.53 & 0.63 & 20.70 & 1.49 & 24.37 & 1.79\\
\cmidrule{1-10}
 & 2 & 100.00 & 0.00 & 83.33 & 23.97 & 95.00 & 15.26 & 95.00 & 15.26\\

 & 10 & 100.00 & 0.00 & 30.33 & 7.65 & 81.33 & 9.37 & 81.67 & 13.15\\

 & 25 & 100.00 & 0.00 & 16.13 & 2.46 & 61.47 & 6.68 & 66.67 & 6.83\\

 & 50 & 100.00 & 0.00 & 8.87 & 1.14 & 43.13 & 3.43 & 47.47 & 4.81\\

\multirow{-5}{*}{\raggedleft\arraybackslash 400} & 100 & 100.00 & 0.00 & 5.03 & 0.61 & 26.53 & 2.08 & 30.63 & 2.65\\
\bottomrule
\end{tabular}
\end{footnotesize}
\end{table}
\begin{table}[tb]
\centering
\caption{\label{tab:constrained}Comparison in terms of mean overlap diversity (\textbf{mean}), its standard deviation (\textbf{std}) and results of Wilcoxon-Mann-Whitney tests (\textbf{stat}). The data is split by the sampling strategy (Single or Poisson).}
\renewcommand{\tabcolsep}{7.2pt}
\renewcommand{\arraystretch}{1.4}
\begin{footnotesize}
\begin{tabular}[t]{rrrrrrrrr}
\toprule
\multicolumn{1}{c}{\textbf{ }} & \multicolumn{1}{c}{\textbf{ }} & \multicolumn{1}{c}{\textbf{ }} & \multicolumn{3}{c}{\textbf{\textcolor{uniform1}{Uniform[1] (1)}}} & \multicolumn{3}{c}{\textbf{\textcolor{poisson}{Poisson (2)}}} \\
\cmidrule(l{3pt}r{3pt}){4-6} \cmidrule(l{3pt}r{3pt}){7-9}
$n$ & $\mu$ & $\alpha$ & \textbf{mean} & \textbf{std} & \textbf{stat} & \textbf{mean} & \textbf{std} & \textbf{stat}\\
\midrule
 &  & 0.05 & \cellcolor{gray!20}{\textbf{21.97}} & 6.17 & \textcolor{poisson}{$\text{2}^{+}$} & 19.12 & 3.96 & \\

 &  & 0.10 & \cellcolor{gray!20}{\textbf{26.73}} & 4.49 & \textcolor{poisson}{$\text{2}^{+}$} & 24.56 & 4.59 & \\

 &  & 0.50 & 47.14 & 4.92 &  & \cellcolor{gray!20}{\textbf{66.80}} & 5.44 & \textcolor{uniform1}{$\text{1}^{+}$}\\

 & \multirow{-4}{*}{\raggedleft\arraybackslash 2} & 1.00 & 62.86 & 5.98 &  & \cellcolor{gray!20}{\textbf{91.29}} & 4.55 & \textcolor{uniform1}{$\text{1}^{+}$}\\

\cline{2-9}
 &  & 0.05 & \cellcolor{gray!20}{\textbf{28.47}} & 1.41 & \textcolor{poisson}{$\text{2}^{+}$} & 22.35 & 1.29 & \\

 &  & 0.10 & \cellcolor{gray!20}{\textbf{36.69}} & 1.41 & \textcolor{poisson}{$\text{2}^{+}$} & 29.97 & 1.61 & \\

 &  & 0.50 & 63.60 & 1.97 &  & \cellcolor{gray!20}{\textbf{64.57}} & 1.06 & \textcolor{uniform1}{$\text{1}^{+}$}\\

 & \multirow{-4}{*}{\raggedleft\arraybackslash 10} & 1.00 & 75.34 & 2.48 &  & \cellcolor{gray!20}{\textbf{83.56}} & 0.64 & \textcolor{uniform1}{$\text{1}^{+}$}\\

\cline{2-9}
 &  & 0.05 & \cellcolor{gray!20}{\textbf{29.15}} & 1.02 & \textcolor{poisson}{$\text{2}^{+}$} & 26.07 & 1.25 & \\

 &  & 0.10 & \cellcolor{gray!20}{\textbf{38.28}} & 0.71 & \textcolor{poisson}{$\text{2}^{+}$} & 35.13 & 1.04 & \\

 &  & 0.50 & \cellcolor{gray!20}{\textbf{68.23}} & 0.81 & \textcolor{poisson}{$\text{2}^{+}$} & 67.43 & 0.56 & \\

\multirow{-12}{*}{\raggedleft\arraybackslash 50} & \multirow{-4}{*}{\raggedleft\arraybackslash 25} & 1.00 & 81.56 & 0.66 &  & \cellcolor{gray!20}{\textbf{83.28}} & 0.38 & \textcolor{uniform1}{$\text{1}^{+}$}\\
\cline{1-9}
 &  & 0.05 & 11.41 & 2.85 &  & \cellcolor{gray!20}{\textbf{11.68}} & 2.58 & \\

 &  & 0.10 & 15.89 & 4.48 &  & \cellcolor{gray!20}{\textbf{19.46}} & 3.59 & \textcolor{uniform1}{$\text{1}^{+}$}\\

 &  & 0.50 & 31.14 & 4.96 &  & \cellcolor{gray!20}{\textbf{58.86}} & 3.44 & \textcolor{uniform1}{$\text{1}^{+}$}\\

 & \multirow{-4}{*}{\raggedleft\arraybackslash 2} & 1.00 & 47.30 & 4.26 &  & \cellcolor{gray!20}{\textbf{85.79}} & 2.69 & \textcolor{uniform1}{$\text{1}^{+}$}\\

\cline{2-9}
 &  & 0.05 & \cellcolor{gray!20}{\textbf{24.68}} & 1.06 & \textcolor{poisson}{$\text{2}^{+}$} & 15.92 & 1.47 & \\

 &  & 0.10 & \cellcolor{gray!20}{\textbf{33.87}} & 1.33 & \textcolor{poisson}{$\text{2}^{+}$} & 24.37 & 1.50 & \\

 &  & 0.50 & \cellcolor{gray!20}{\textbf{59.10}} & 2.51 &  & 59.07 & 1.05 & \\

 & \multirow{-4}{*}{\raggedleft\arraybackslash 10} & 1.00 & 70.94 & 3.33 &  & \cellcolor{gray!20}{\textbf{80.41}} & 0.72 & \textcolor{uniform1}{$\text{1}^{+}$}\\

\cline{2-9}
 &  & 0.05 & \cellcolor{gray!20}{\textbf{26.24}} & 0.61 & \textcolor{poisson}{$\text{2}^{+}$} & 20.53 & 0.92 & \\

 &  & 0.10 & \cellcolor{gray!20}{\textbf{35.98}} & 0.71 & \textcolor{poisson}{$\text{2}^{+}$} & 29.57 & 0.96 & \\

 &  & 0.50 & \cellcolor{gray!20}{\textbf{66.55}} & 0.68 & \textcolor{poisson}{$\text{2}^{+}$} & 63.75 & 0.57 & \\

 & \multirow{-4}{*}{\raggedleft\arraybackslash 25} & 1.00 & 79.78 & 0.69 &  & \cellcolor{gray!20}{\textbf{81.13}} & 0.45 & \textcolor{uniform1}{$\text{1}^{+}$}\\

\cline{2-9}
 &  & 0.05 & \cellcolor{gray!20}{\textbf{26.69}} & 0.44 & \textcolor{poisson}{$\text{2}^{+}$} & 22.90 & 0.77 & \\

 &  & 0.10 & \cellcolor{gray!20}{\textbf{36.64}} & 0.40 & \textcolor{poisson}{$\text{2}^{+}$} & 32.16 & 0.64 & \\

 &  & 0.50 & \cellcolor{gray!20}{\textbf{67.86}} & 0.35 & \textcolor{poisson}{$\text{2}^{+}$} & 66.30 & 0.43 & \\

\multirow{-16}{*}{\raggedleft\arraybackslash 100} & \multirow{-4}{*}{\raggedleft\arraybackslash 50} & 1.00 & 81.74 & 0.46 &  & \cellcolor{gray!20}{\textbf{82.17}} & 0.20 & \textcolor{uniform1}{$\text{1}^{+}$}\\
\bottomrule
\end{tabular}
\end{footnotesize}
\end{table}
For MSTs on complete graphs the number of leafs, the maximum degree and the diameter can take $n-2$ different values in $\{2, 3, \ldots, n-1\}$ where in all three cases $n$-vertex paths and star graphs pose the extreme cases; i.e. a star graph has $(n-1)$ leafs, maximum degree $(n-1)$ and diameter $2$ while a path has diameter $n-1$ with maximum degree $2$ and $2$ leafs.
\begin{table}[ht!]
\centering
\caption{\label{tab:constrained_diversity_measures}Diversity (in percent of the maximally achievable value) with respect to overlap diversity $D_o$, maximum degree (\textbf{Max. degree}), number of leafs (\textbf{Leaf}) and the diameter (\textbf{Diameter}) in the constrained setting.}
\renewcommand{\tabcolsep}{4pt}
\renewcommand{\arraystretch}{1.38}
\begin{footnotesize}
\begin{tabular}[t]{rrrrrrrrrrr}
\toprule
\multicolumn{1}{c}{\textbf{ }} & \multicolumn{1}{c}{\textbf{ }} & \multicolumn{1}{c}{\textbf{ }} & \multicolumn{2}{c}{\textbf{$D_o$}} & \multicolumn{2}{c}{\textbf{Max. degree}} & \multicolumn{2}{c}{\textbf{Leaf}} & \multicolumn{2}{c}{\textbf{Diameter}} \\
\cmidrule(l{3pt}r{3pt}){4-5} \cmidrule(l{3pt}r{3pt}){6-7} \cmidrule(l{3pt}r{3pt}){8-9} \cmidrule(l{3pt}r{3pt}){10-11}
$n$ & $\mu$ & $\alpha$ & \textbf{mean} & \textbf{std} & \textbf{mean} & \textbf{std} & \textbf{mean} & \textbf{std} & \textbf{mean} & \textbf{std}\\
\midrule
 &  & 0.05 & 19.12 & 3.96 & 75.00 & 25.43 & 88.33 & 21.51 & 95.00 & 15.26\\

 &  & 0.10 & 24.56 & 4.59 & 75.00 & 25.43 & 95.00 & 15.26 & 96.67 & 12.69\\

 &  & 0.50 & 66.80 & 5.44 & 75.00 & 25.43 & 93.33 & 17.29 & 96.67 & 12.69\\

 & \multirow{-4}{*}{\raggedleft\arraybackslash 2} & 1.00 & 91.29 & 4.55 & 80.00 & 24.91 & 88.33 & 21.51 & 88.33 & 21.51\\

\cline{2-11}
 &  & 0.05 & 22.35 & 1.29 & 18.33 & 5.92 & 42.00 & 9.97 & 43.00 & 11.49\\

 &  & 0.10 & 29.97 & 1.61 & 22.67 & 7.85 & 44.00 & 9.32 & 55.67 & 11.35\\

 &  & 0.50 & 64.57 & 1.06 & 27.33 & 6.40 & 57.67 & 8.98 & 66.67 & 10.93\\

 & \multirow{-4}{*}{\raggedleft\arraybackslash 10} & 1.00 & 83.56 & 0.64 & 29.67 & 6.69 & 56.67 & 9.22 & 66.33 & 9.28\\

\cline{2-11}
 &  & 0.05 & 26.07 & 1.25 & 9.60 & 1.99 & 20.13 & 4.26 & 23.47 & 5.22\\

 &  & 0.10 & 35.13 & 1.04 & 10.93 & 2.08 & 23.60 & 4.74 & 26.53 & 4.87\\

 &  & 0.50 & 67.43 & 0.56 & 12.93 & 3.10 & 32.93 & 4.89 & 36.67 & 3.65\\

\multirow{-12}{*}{\raggedleft\arraybackslash 50} & \multirow{-4}{*}{\raggedleft\arraybackslash 25} & 1.00 & 83.28 & 0.38 & 14.27 & 2.50 & 34.80 & 5.37 & 39.87 & 5.51\\
\cmidrule{1-11}
 &  & 0.05 & 11.68 & 2.58 & 65.00 & 23.30 & 90.00 & 20.34 & 96.67 & 12.69\\

 &  & 0.10 & 19.46 & 3.59 & 65.00 & 23.30 & 96.67 & 12.69 & 91.67 & 18.95\\

 &  & 0.50 & 58.86 & 3.44 & 81.67 & 24.51 & 96.67 & 12.69 & 98.33 & 9.13\\

 & \multirow{-4}{*}{\raggedleft\arraybackslash 2} & 1.00 & 85.79 & 2.69 & 83.33 & 23.97 & 91.67 & 18.95 & 91.67 & 18.95\\

\cline{2-11}
 &  & 0.05 & 15.92 & 1.47 & 17.00 & 5.96 & 44.33 & 10.40 & 46.00 & 11.02\\

 &  & 0.10 & 24.37 & 1.50 & 20.33 & 6.15 & 56.00 & 8.14 & 62.33 & 10.40\\

 &  & 0.50 & 59.07 & 1.05 & 26.00 & 5.63 & 61.00 & 11.25 & 80.00 & 7.88\\

 & \multirow{-4}{*}{\raggedleft\arraybackslash 10} & 1.00 & 80.41 & 0.72 & 32.00 & 6.64 & 67.00 & 11.19 & 75.33 & 10.08\\

\cline{2-11}
 &  & 0.05 & 20.53 & 0.92 & 7.73 & 3.14 & 23.87 & 4.52 & 27.73 & 4.32\\

 &  & 0.10 & 29.57 & 0.96 & 9.07 & 2.33 & 27.47 & 5.82 & 33.60 & 6.36\\

 &  & 0.50 & 63.75 & 0.57 & 12.80 & 1.63 & 40.27 & 5.94 & 53.20 & 5.96\\

 & \multirow{-4}{*}{\raggedleft\arraybackslash 25} & 1.00 & 81.13 & 0.45 & 14.67 & 2.64 & 41.87 & 4.67 & 52.00 & 6.73\\
\cline{2-11}
 &  & 0.05 & 22.90 & 0.77 & 4.13 & 1.48 & 14.47 & 2.91 & 16.40 & 2.99\\

 &  & 0.10 & 32.16 & 0.64 & 4.47 & 1.01 & 17.33 & 3.08 & 21.20 & 3.81\\

 &  & 0.50 & 66.30 & 0.43 & 6.53 & 0.90 & 24.40 & 3.17 & 31.47 & 2.83\\

\multirow{-16}{*}{\raggedleft\arraybackslash 100} & \multirow{-4}{*}{\raggedleft\arraybackslash 50} & 1.00 & 82.17 & 0.20 & 7.60 & 1.33 & 26.20 & 3.17 & 33.67 & 2.41\\
\bottomrule
\end{tabular}
\end{footnotesize}
\end{table}
Table~\ref{tab:unconstrained_diversity_measures} shows mean and standard deviation of diversity (as percentage values) with respect to the maximum degree, the number of leafs and the diameter in the unconstrained case. Due to space limitations we show the statistics for the Poisson-sampling only. For these features diversity is measured as the fraction of distinct values in the final populations. E.g. maximum max-degree diversity is obtained if all trees have different maximum degree.\footnote{Note that in our setup $\mu \leq \lfloor\frac{n}{2}\rfloor < n-2$ and hence maximum diversity could be possible. However, for $\mu$ close to the maximum value we actually do not know if this is achievable in theory.} The values in the table indicate a clear trend: with increasing $\mu$ the mean diversity values decrease. This trend is particularly strong for the maximum degree. Note that this is in line with Corollary~\ref{cor:existence_population_edge_disjoint_trees} where the proof was constructive and the resulting population was a set of edge-disjoint paths all with maximum degree~2. The data suggests, that in order to obtain higher diversity with respect to, e.g. maximum degree, one likely has to sacrifice overlap diversity. This is clearly a multi-objective problem and will be part of our future agenda.


\section{Constrained Diversity Optimization}
\label{sec:sec5}

We now turn our focus to constrained diversity optimization on random Euclidean graphs where nodes are associated with point coordinates in the Euclidean plane. To this end we consider $n \in \{50, 100, 200, 400\}$, $\mu \in \{2, 10, 25, 50, 100, 200\}$ such that $\mu \leq \lfloor\frac{n}{2}\rfloor$ and $\alpha \in \{0.05, 0.10, 0.5, 1\}$. Consistent with the setup in the unconstrained case \EA is run 30 times on each generated instance for at most $\mu n^2$ iterations.
Table~\ref{tab:constrained} shows the results following the format of Table~\ref{tab:uniform_weights}. We show a subset of the results for Uniform[1] and Poisson sampling for $n \in \{50,100\}$; the omitted data shows the same picture.
As expected diversity in most cases increases with increasing $\alpha$ as there is more flexibility in choosing alternative edges (see Figure~\ref{fig:superimposed_edges} for a superimposed visualisation of solutions).

In addition, a distinctive pattern with respect to the sampling strategies can be observed. For low $\alpha$, performing single 1-EX moves seems advantageous. In contrast, performing multiple 1-EX moves mutation leads to better results in terms of overlap diversity for increasing $\alpha$. This seems plausible and in par with our observations in Section~\ref{sec:sec4}: as $\alpha$ increases, so does the size of the feasible search space. Hence, stronger mutation may explore this space more thoroughly as multiple subsequent edge exchanges might be necessary in order to come up with a solution that in fact adheres to the quality constraint. However, this effect diminishes with $\mu$ approaching $\lfloor\frac{n}{2}\rfloor$. We see that single mutation steps become more effective and the gap between single step and multi step mutation decreases even in the case where multi-step mutation is still stronger.

Table~\ref{tab:constrained_diversity_measures} explores diversity with respect to the maximum degree, the number of leafs and the diameter for the constrained setting. For fixed $\alpha$ diversity with respect to all three features decreases with increasing $\mu$ showing an inverse trend in comparison with overlap diversity. In line with the unconstrained case the maximum degree diversity is consistently lowest reaching below $15\%$ for $\mu=\lfloor\frac{n}{2}\rfloor$. Keeping $\mu$ fixed we see a trend towards higher diversity values for increasing $\alpha$. This seems plausible as for low $\alpha$ the final population differs in only few edges (high overlap) and hence the STs are also more likely to be structurally similar.


\section{Conclusion}
\label{sec:sec6}

We studied the minimum spanning tree problem in the context of evolutionary diversity optimization for the first time. Here, the goal is to evolve a population of high quality solutions which all satisfy minimum quality requirements. We studied a baseline $(\mu+1)$ evolutionary algorithm with a diversity measure based on pairwise edge overlap of solutions. Runtime complexity results show that a $(2+1)$-EA is able to find two maximally diverse spanning trees in expected time $O(n^3\log n)$ in the unconstrained setting where $n$ is the number of nodes; one of the first runtime results for evolutionary diversity optimization on well-known combinatorial optimization problems. Complementary experiments suggest that $(\mu+1)$-EA with $\mu$ up to $\lfloor\frac{n}{2}\rfloor$ can efficiently evolve diverse populations of edge-disjoint spanning trees in the unconstrained and the constrained case where in the latter a relaxation of the quality constrained allows for higher diversity. However, increasing population size leads to less diversity with of solutions with respect to other desired tree features: the maximum degree, the number of leafs and the diameter.

Future research endeavors will focus on improving and generalizing the theoretical runtime bounds to strengthen the still rare theoretical foundations of evolutionary diversity optimization on combinatorial optimization problems. Here, in particular an extension towards the constrained case seems interesting and challenging.
In addition, multi-objective approaches which aim to simultaneously maximise several diversity measures appear promising.

\section*{Acknowledgment}
This work was supported by the Australian Research Council through grant DP190103894.

\bibliographystyle{ACM-Reference-Format}
\bibliography{arxiv}


\begin{thebibliography}{32}


\ifx \showCODEN    \undefined \def \showCODEN     #1{\unskip}     \fi
\ifx \showDOI      \undefined \def \showDOI       #1{#1}\fi
\ifx \showISBNx    \undefined \def \showISBNx     #1{\unskip}     \fi
\ifx \showISBNxiii \undefined \def \showISBNxiii  #1{\unskip}     \fi
\ifx \showISSN     \undefined \def \showISSN      #1{\unskip}     \fi
\ifx \showLCCN     \undefined \def \showLCCN      #1{\unskip}     \fi
\ifx \shownote     \undefined \def \shownote      #1{#1}          \fi
\ifx \showarticletitle \undefined \def \showarticletitle #1{#1}   \fi
\ifx \showURL      \undefined \def \showURL       {\relax}        \fi
\providecommand\bibfield[2]{#2}
\providecommand\bibinfo[2]{#2}
\providecommand\natexlab[1]{#1}
\providecommand\showeprint[2][]{arXiv:#2}

\bibitem[\protect\citeauthoryear{Alexander, Kortman, and Neumann}{Alexander
  et~al\mbox{.}}{2017}]%
        {DBLP:conf/gecco/AlexanderKN17}
\bibfield{author}{\bibinfo{person}{Bradley Alexander}, \bibinfo{person}{James
  Kortman}, {and} \bibinfo{person}{Aneta Neumann}.}
  \bibinfo{year}{2017}\natexlab{}.
\newblock \showarticletitle{Evolution of artistic image variants through
  feature based diversity optimisation}. In
  \bibinfo{booktitle}{\emph{Proceedings of the Genetic and Evolutionary
  Computation Conference, {GECCO}}} \emph{(\bibinfo{series}{GECCO '17})},
  \bibfield{editor}{\bibinfo{person}{Peter A.~N. Bosman}} (Ed.).
  \bibinfo{publisher}{{ACM}}, \bibinfo{pages}{171--178}.
\newblock
\urldef\tempurl%
\url{https://doi.org/10.1145/3071178.3071342}
\showDOI{\tempurl}


\bibitem[\protect\citeauthoryear{Alspach, Bermond, and Sotteau}{Alspach
  et~al\mbox{.}}{1990}]%
        {Alspach1990}
\bibfield{author}{\bibinfo{person}{Brian Alspach}, \bibinfo{person}{Jean-Claude
  Bermond}, {and} \bibinfo{person}{Dominique Sotteau}.}
  \bibinfo{year}{1990}\natexlab{}.
\newblock \bibinfo{booktitle}{\emph{Decomposition into cycles I: Hamilton
  decompositions}}.
\newblock \bibinfo{publisher}{Springer Netherlands}, \bibinfo{pages}{9--18}.
\newblock
\showISBNx{978-94-009-0517-7}


\bibitem[\protect\citeauthoryear{Bossek, Grimme, and Neumann}{Bossek
  et~al\mbox{.}}{2019a}]%
        {BNG2019}
\bibfield{author}{\bibinfo{person}{Jakob Bossek}, \bibinfo{person}{Christian
  Grimme}, {and} \bibinfo{person}{Frank Neumann}.}
  \bibinfo{year}{2019}\natexlab{a}.
\newblock \showarticletitle{{On the benefits of biased edge-exchange mutation
  for the multi-criteria spanning tree problem}}. In
  \bibinfo{booktitle}{\emph{Proceedings of the 2019 Genetic and Evolutionary
  Computation Conference}} \emph{(\bibinfo{series}{GECCO '19})}.
  \bibinfo{publisher}{ACM Press}, \bibinfo{address}{New York, New York, USA},
  \bibinfo{pages}{516--523}.
\newblock
\showISBNx{9781450361118}
\urldef\tempurl%
\url{https://doi.org/10.1145/3321707.3321818}
\showDOI{\tempurl}


\bibitem[\protect\citeauthoryear{Bossek, Kerschke, Neumann, Wagner, Neumann,
  and Trautmann}{Bossek et~al\mbox{.}}{2019b}]%
        {Bossek2019Evolving}
\bibfield{author}{\bibinfo{person}{Jakob Bossek}, \bibinfo{person}{Pascal
  Kerschke}, \bibinfo{person}{Aneta Neumann}, \bibinfo{person}{Markus Wagner},
  \bibinfo{person}{Frank Neumann}, {and} \bibinfo{person}{Heike Trautmann}.}
  \bibinfo{year}{2019}\natexlab{b}.
\newblock \showarticletitle{{Evolving diverse TSP instances by means of novel
  and creative mutation operators}}. In \bibinfo{booktitle}{\emph{Proceedings
  of the 15th ACM/SIGEVO Conference on Foundations of Genetic Algorithms}}
  \emph{(\bibinfo{series}{FOGA XV})}. \bibinfo{publisher}{ACM Press},
  \bibinfo{address}{New York, New York, USA}, \bibinfo{pages}{58--71}.
\newblock
\showISBNx{9781450362542}
\urldef\tempurl%
\url{https://doi.org/10.1145/3299904.3340307}
\showDOI{\tempurl}


\bibitem[\protect\citeauthoryear{Chartrand and Lesniak}{Chartrand and
  Lesniak}{1986}]%
        {Chartrand10986GraphsAndDigraphs}
\bibfield{author}{\bibinfo{person}{Gary Chartrand} {and} \bibinfo{person}{Linda
  Lesniak}.} \bibinfo{year}{1986}\natexlab{}.
\newblock \bibinfo{booktitle}{\emph{Graphs \& Digraphs (2nd Ed.)}}.
\newblock \bibinfo{publisher}{Wadsworth Publ. Co.}, \bibinfo{address}{USA}.
\newblock
\showISBNx{0534063241}


\bibitem[\protect\citeauthoryear{Cormen, Leiserson, Rivest, and Stein}{Cormen
  et~al\mbox{.}}{2001}]%
        {cormen01introduction}
\bibfield{author}{\bibinfo{person}{Thomas~H. Cormen},
  \bibinfo{person}{Charles~E. Leiserson}, \bibinfo{person}{Ronald~L. Rivest},
  {and} \bibinfo{person}{Clifford Stein}.} \bibinfo{year}{2001}\natexlab{}.
\newblock \bibinfo{booktitle}{\emph{Introduction to Algorithms}
  (\bibinfo{edition}{2nd} ed.)}.
\newblock \bibinfo{publisher}{The MIT Press}.
\newblock
\showISBNx{0262032937}


\bibitem[\protect\citeauthoryear{Dang, Jansen, and Lehre}{Dang
  et~al\mbox{.}}{2017}]%
        {DBLP:journals/algorithmica/DangJL17}
\bibfield{author}{\bibinfo{person}{Duc{-}Cuong Dang}, \bibinfo{person}{Thomas
  Jansen}, {and} \bibinfo{person}{Per~Kristian Lehre}.}
  \bibinfo{year}{2017}\natexlab{}.
\newblock \showarticletitle{Populations Can Be Essential in Tracking Dynamic
  Optima}.
\newblock \bibinfo{journal}{\emph{Algorithmica}} \bibinfo{volume}{78},
  \bibinfo{number}{2} (\bibinfo{year}{2017}), \bibinfo{pages}{660--680}.
\newblock


\bibitem[\protect\citeauthoryear{Do, Bossek, Neumann, and Neumann}{Do
  et~al\mbox{.}}{2020}]%
        {Do2020EvolvingDiverseTSPTours}
\bibfield{author}{\bibinfo{person}{Anh~Viet Do}, \bibinfo{person}{Jakob
  Bossek}, \bibinfo{person}{Aneta Neumann}, {and} \bibinfo{person}{Frank
  Neumann}.} \bibinfo{year}{2020}\natexlab{}.
\newblock \showarticletitle{Evolving Diverse Sets of Tours for the Travelling
  Salesperson Problem}. In \bibinfo{booktitle}{\emph{Proceedings of the 2020
  Genetic and Evolutionary Computation Conference}} (Canc\'{u}n, Mexico)
  \emph{(\bibinfo{series}{GECCO '20})}. \bibinfo{publisher}{Association for
  Computing Machinery}, \bibinfo{address}{New York, NY, USA},
  \bibinfo{pages}{681–689}.
\newblock
\showISBNx{9781450371285}
\urldef\tempurl%
\url{https://doi.org/10.1145/3377930.3389844}
\showDOI{\tempurl}


\bibitem[\protect\citeauthoryear{{Doerr, B., Neumann, F. (Eds.)}}{{Doerr, B.,
  Neumann, F. (Eds.)}}{2020}]%
        {BookDoeNeu}
\bibfield{author}{\bibinfo{person}{{Doerr, B., Neumann, F. (Eds.)}}.}
  \bibinfo{year}{2020}\natexlab{}.
\newblock \bibinfo{booktitle}{\emph{Theory of Evolutionary Computation --
  Recent Developments in Discrete Optimization}}.
\newblock \bibinfo{publisher}{Springer}.
\newblock
\showISBNx{978-3-030-29413-7}
\urldef\tempurl%
\url{https://doi.org/10.1007/978-3-030-29414-4}
\showDOI{\tempurl}


\bibitem[\protect\citeauthoryear{Friedrich and Neumann}{Friedrich and
  Neumann}{2015}]%
        {DBLP:journals/ec/FriedrichN15}
\bibfield{author}{\bibinfo{person}{Tobias Friedrich} {and}
  \bibinfo{person}{Frank Neumann}.} \bibinfo{year}{2015}\natexlab{}.
\newblock \showarticletitle{Maximizing Submodular Functions under Matroid
  Constraints by Evolutionary Algorithms}.
\newblock \bibinfo{journal}{\emph{Evolutionary Computation}}
  \bibinfo{volume}{23}, \bibinfo{number}{4} (\bibinfo{year}{2015}),
  \bibinfo{pages}{543--558}.
\newblock


\bibitem[\protect\citeauthoryear{Gao, Nallaperuma, and Neumann}{Gao
  et~al\mbox{.}}{2016}]%
        {DBLP:conf/ppsn/GaoNN16}
\bibfield{author}{\bibinfo{person}{Wanru Gao}, \bibinfo{person}{Samadhi
  Nallaperuma}, {and} \bibinfo{person}{Frank Neumann}.}
  \bibinfo{year}{2016}\natexlab{}.
\newblock \showarticletitle{Feature-based diversity optimization for problem
  instance classification}. In \bibinfo{booktitle}{\emph{Parallel Problem
  Solving from Nature, PPSN}} \emph{(\bibinfo{series}{LNCS},
  Vol.~\bibinfo{volume}{9921})}. \bibinfo{publisher}{Springer},
  \bibinfo{pages}{869--879}.
\newblock
\showISBNx{978-3-319-45822-9}
\urldef\tempurl%
\url{https://doi.org/10.1007/978-3-319-45823-6\_81}
\showDOI{\tempurl}


\bibitem[\protect\citeauthoryear{Gao and Neumann}{Gao and Neumann}{2014}]%
        {DBLP:conf/gecco/GaoN14}
\bibfield{author}{\bibinfo{person}{Wanru Gao} {and} \bibinfo{person}{Frank
  Neumann}.} \bibinfo{year}{2014}\natexlab{}.
\newblock \showarticletitle{Runtime analysis for maximizing population
  diversity in single-objective optimization}. In
  \bibinfo{booktitle}{\emph{Proceedings of the 2014 Genetic and Evolutionary
  Computation Conference}} \emph{(\bibinfo{series}{GECCO '14})}.
  \bibinfo{publisher}{{ACM}}, \bibinfo{pages}{777--784}.
\newblock


\bibitem[\protect\citeauthoryear{Gao, Pourhassan, and Neumann}{Gao
  et~al\mbox{.}}{2015}]%
        {DBLP:conf/gecco/GaoPN15}
\bibfield{author}{\bibinfo{person}{Wanru Gao}, \bibinfo{person}{Mojgan
  Pourhassan}, {and} \bibinfo{person}{Frank Neumann}.}
  \bibinfo{year}{2015}\natexlab{}.
\newblock \showarticletitle{Runtime Analysis of Evolutionary Diversity
  Optimization and the Vertex Cover Problem}. In
  \bibinfo{booktitle}{\emph{Proceedings of the 2015 Genetic and Evolutionary
  Computation Conference (Companion)}} \emph{(\bibinfo{series}{GECCO '15})}.
  \bibinfo{publisher}{{ACM}}, \bibinfo{pages}{1395--1396}.
\newblock


\bibitem[\protect\citeauthoryear{Gouveia and Magnanti}{Gouveia and
  Magnanti}{2003}]%
        {GM2003NetworkFlowModels}
\bibfield{author}{\bibinfo{person}{Luis Gouveia} {and}
  \bibinfo{person}{Thomas~L. Magnanti}.} \bibinfo{year}{2003}\natexlab{}.
\newblock \showarticletitle{Network flow models for designing
  diameter-constrained minimum-spanning and Steiner trees}.
\newblock \bibinfo{journal}{\emph{Networks}} \bibinfo{volume}{41},
  \bibinfo{number}{3} (\bibinfo{year}{2003}), \bibinfo{pages}{159--173}.
\newblock
\urldef\tempurl%
\url{https://doi.org/10.1002/net.10069}
\showDOI{\tempurl}
\showeprint{https://onlinelibrary.wiley.com/doi/pdf/10.1002/net.10069}


\bibitem[\protect\citeauthoryear{Jansen}{Jansen}{2013}]%
        {ncs/Jansen13}
\bibfield{author}{\bibinfo{person}{Thomas Jansen}.}
  \bibinfo{year}{2013}\natexlab{}.
\newblock \bibinfo{booktitle}{\emph{Analyzing Evolutionary Algorithms - The
  Computer Science Perspective}}.
\newblock \bibinfo{publisher}{Springer}. 1--236 pages.
\newblock
\showISBNx{978-3-642-17338-7}


\bibitem[\protect\citeauthoryear{Kerschke, Hoos, Neumann, and
  Trautmann}{Kerschke et~al\mbox{.}}{2019}]%
        {DBLP:journals/ec/KerschkeHNT19}
\bibfield{author}{\bibinfo{person}{Pascal Kerschke}, \bibinfo{person}{Holger~H.
  Hoos}, \bibinfo{person}{Frank Neumann}, {and} \bibinfo{person}{Heike
  Trautmann}.} \bibinfo{year}{2019}\natexlab{}.
\newblock \showarticletitle{Automated Algorithm Selection: Survey and
  Perspectives}.
\newblock \bibinfo{journal}{\emph{Evolutionary Computation}}
  \bibinfo{volume}{27}, \bibinfo{number}{1} (\bibinfo{year}{2019}),
  \bibinfo{pages}{3--45}.
\newblock
\urldef\tempurl%
\url{https://doi.org/10.1162/evco\_a\_00242}
\showDOI{\tempurl}


\bibitem[\protect\citeauthoryear{Kruskal}{Kruskal}{1956}]%
        {Kruskal56}
\bibfield{author}{\bibinfo{person}{Joseph~B. Kruskal}.}
  \bibinfo{year}{1956}\natexlab{}.
\newblock \showarticletitle{On the Shortest Spanning Subtree of a Graph and the
  Traveling Salesman Problem}.
\newblock \bibinfo{journal}{\emph{Proc. Amer. Math. Soc.}} \bibinfo{volume}{7},
  \bibinfo{number}{1} (\bibinfo{year}{1956}), \bibinfo{pages}{48--50}.
\newblock
\showISSN{00029939, 10886826}


\bibitem[\protect\citeauthoryear{Lang, Bischl, and Surmann}{Lang
  et~al\mbox{.}}{2017}]%
        {Rbatchtools}
\bibfield{author}{\bibinfo{person}{Michel Lang}, \bibinfo{person}{Bernd
  Bischl}, {and} \bibinfo{person}{Dirk Surmann}.}
  \bibinfo{year}{2017}\natexlab{}.
\newblock \showarticletitle{batchtools: Tools for R to work on batch systems}.
\newblock \bibinfo{journal}{\emph{The Journal of Open Source Software}}
  \bibinfo{volume}{2}, \bibinfo{number}{10} (\bibinfo{date}{feb}
  \bibinfo{year}{2017}).
\newblock
\urldef\tempurl%
\url{https://doi.org/10.21105/joss.00135}
\showDOI{\tempurl}


\bibitem[\protect\citeauthoryear{Lehman and Stanley}{Lehman and
  Stanley}{2011}]%
        {lehman2011evolving}
\bibfield{author}{\bibinfo{person}{Joel Lehman} {and}
  \bibinfo{person}{Kenneth~O Stanley}.} \bibinfo{year}{2011}\natexlab{}.
\newblock \showarticletitle{Evolving a diversity of virtual creatures through
  novelty search and local competition}. In
  \bibinfo{booktitle}{\emph{Proceedings of the 2011 Genetic and Evolutionary
  Computation Conference}}. ACM, \bibinfo{pages}{211--218}.
\newblock


\bibitem[\protect\citeauthoryear{Neumann, Gao, Doerr, Neumann, and
  Wagner}{Neumann et~al\mbox{.}}{2018}]%
        {DBLP:journals/corr/abs-1802-05448}
\bibfield{author}{\bibinfo{person}{Aneta Neumann}, \bibinfo{person}{Wanru Gao},
  \bibinfo{person}{Carola Doerr}, \bibinfo{person}{Frank Neumann}, {and}
  \bibinfo{person}{Markus Wagner}.} \bibinfo{year}{2018}\natexlab{}.
\newblock \showarticletitle{Discrepancy-based evolutionary diversity
  optimization}. In \bibinfo{booktitle}{\emph{Proceedings of the 2018 Genetic
  and Evolutionary Computation Conference}} \emph{(\bibinfo{series}{GECCO
  '18})}. \bibinfo{publisher}{{ACM}}, \bibinfo{pages}{991--998}.
\newblock
\urldef\tempurl%
\url{https://doi.org/10.1145/3205455.3205532}
\showDOI{\tempurl}


\bibitem[\protect\citeauthoryear{Neumann, Gao, Wagner, and Neumann}{Neumann
  et~al\mbox{.}}{2019}]%
        {DBLP:conf/gecco/NeumannG0019}
\bibfield{author}{\bibinfo{person}{Aneta Neumann}, \bibinfo{person}{Wanru Gao},
  \bibinfo{person}{Markus Wagner}, {and} \bibinfo{person}{Frank Neumann}.}
  \bibinfo{year}{2019}\natexlab{}.
\newblock \showarticletitle{Evolutionary diversity optimization using
  multi-objective indicators}. In \bibinfo{booktitle}{\emph{Proceedings of the
  2019 Genetic and Evolutionary Computation Conference}}
  \emph{(\bibinfo{series}{GECCO '19})}. \bibinfo{publisher}{{ACM}},
  \bibinfo{pages}{837--845}.
\newblock
\showISBNx{978-1-4503-6111-8}
\urldef\tempurl%
\url{https://doi.org/10.1145/3321707.3321796}
\showDOI{\tempurl}


\bibitem[\protect\citeauthoryear{Neumann and Neumann}{Neumann and
  Neumann}{2020}]%
        {DBLP:conf/ppsn/NeumannN20}
\bibfield{author}{\bibinfo{person}{Aneta Neumann} {and} \bibinfo{person}{Frank
  Neumann}.} \bibinfo{year}{2020}\natexlab{}.
\newblock \showarticletitle{Optimising Monotone Chance-Constrained Submodular
  Functions Using Evolutionary Multi-objective Algorithms}. In
  \bibinfo{booktitle}{\emph{Proceedings of the Parallel Problem Solving from
  Nature, {PPSN} 2020}} \emph{(\bibinfo{series}{Lecture Notes in Computer
  Science}, Vol.~\bibinfo{volume}{12269})}. \bibinfo{publisher}{Springer},
  \bibinfo{pages}{404--417}.
\newblock


\bibitem[\protect\citeauthoryear{Neumann and Witt}{Neumann and Witt}{2010}]%
        {DBLP:books/daglib/0025643}
\bibfield{author}{\bibinfo{person}{Frank Neumann} {and}
  \bibinfo{person}{Carsten Witt}.} \bibinfo{year}{2010}\natexlab{}.
\newblock \bibinfo{booktitle}{\emph{Bioinspired Computation in Combinatorial
  Optimization}}.
\newblock \bibinfo{publisher}{Springer}.
\newblock
\showISBNx{978-3-642-16543-6}
\urldef\tempurl%
\url{https://doi.org/10.1007/978-3-642-16544-3}
\showDOI{\tempurl}


\bibitem[\protect\citeauthoryear{Prim}{Prim}{1957}]%
        {Prim57}
\bibfield{author}{\bibinfo{person}{R.~C. Prim}.}
  \bibinfo{year}{1957}\natexlab{}.
\newblock \showarticletitle{Shortest Connection Networks And Some
  Generalizations}.
\newblock \bibinfo{journal}{\emph{Bell System Technical Journal}}
  \bibinfo{volume}{36}, \bibinfo{number}{6} (\bibinfo{year}{1957}),
  \bibinfo{pages}{1389--1401}.
\newblock
\showISSN{1538-7305}
\urldef\tempurl%
\url{https://doi.org/10.1002/j.1538-7305.1957.tb01515.x}
\showDOI{\tempurl}


\bibitem[\protect\citeauthoryear{Pugh, Soros, and Stanley}{Pugh
  et~al\mbox{.}}{2016}]%
        {pugh2016extended}
\bibfield{author}{\bibinfo{person}{Justin~K Pugh}, \bibinfo{person}{Lisa~B
  Soros}, {and} \bibinfo{person}{Kenneth~O Stanley}.}
  \bibinfo{year}{2016}\natexlab{}.
\newblock \showarticletitle{An Extended Study of Quality Diversity Algorithms}.
  In \bibinfo{booktitle}{\emph{Proceedings of the 2016 on Genetic and
  Evolutionary Computation Conference Companion}}. ACM,
  \bibinfo{pages}{19--20}.
\newblock


\bibitem[\protect\citeauthoryear{Qian, Shi, Yu, and Tang}{Qian
  et~al\mbox{.}}{2017}]%
        {DBLP:conf/ijcai/QianSYT17}
\bibfield{author}{\bibinfo{person}{Chao Qian}, \bibinfo{person}{Jing{-}Cheng
  Shi}, \bibinfo{person}{Yang Yu}, {and} \bibinfo{person}{Ke Tang}.}
  \bibinfo{year}{2017}\natexlab{}.
\newblock \showarticletitle{On Subset Selection with General Cost Constraints}.
  In \bibinfo{booktitle}{\emph{Proceedings of the 2017 International Joint
  Conference on Artificial Intelligence}} \emph{(\bibinfo{series}{IJCAI '19})}.
  \bibinfo{pages}{2613--2619}.
\newblock


\bibitem[\protect\citeauthoryear{Qian, Yu, and Zhou}{Qian
  et~al\mbox{.}}{2015}]%
        {DBLP:conf/nips/QianYZ15}
\bibfield{author}{\bibinfo{person}{Chao Qian}, \bibinfo{person}{Yang Yu}, {and}
  \bibinfo{person}{Zhi{-}Hua Zhou}.} \bibinfo{year}{2015}\natexlab{}.
\newblock \showarticletitle{Subset Selection by {P}areto Optimization}. In
  \bibinfo{booktitle}{\emph{Proceedings of the 28th International Conference on
  Neural Information Processing Systems - Volume 1, NIPS 2015}}.
  \bibinfo{pages}{1774--1782}.
\newblock


\bibitem[\protect\citeauthoryear{{R Core Team}}{{R Core Team}}{2019}]%
        {Rlang}
\bibfield{author}{\bibinfo{person}{{R Core Team}}.}
  \bibinfo{year}{2019}\natexlab{}.
\newblock \bibinfo{booktitle}{\emph{R: A Language and Environment for
  Statistical Computing}}.
\newblock R Foundation for Statistical Computing, Vienna, Austria.
\newblock
\urldef\tempurl%
\url{https://www.R-project.org/}
\showURL{%
\tempurl}


\bibitem[\protect\citeauthoryear{Raidl, Koller, and Julstrom}{Raidl
  et~al\mbox{.}}{2006}]%
        {RKJ2006_BiasedMutationOperators}
\bibfield{author}{\bibinfo{person}{G{\"{u}}nther~R Raidl},
  \bibinfo{person}{Gabriele Koller}, {and} \bibinfo{person}{Bryant~A
  Julstrom}.} \bibinfo{year}{2006}\natexlab{}.
\newblock \showarticletitle{{Biased Mutation Operators for Subgraph-Selection
  Problems}}.
\newblock \bibinfo{journal}{\emph{IEEE Transactions on Evolutionary
  Computation}} \bibinfo{volume}{10}, \bibinfo{number}{2}
  (\bibinfo{year}{2006}), \bibinfo{pages}{145--156}.
\newblock


\bibitem[\protect\citeauthoryear{Roostapour, Neumann, Neumann, and
  Friedrich}{Roostapour et~al\mbox{.}}{2019}]%
        {DBLP:conf/aaai/RoostapourN0019}
\bibfield{author}{\bibinfo{person}{Vahid Roostapour}, \bibinfo{person}{Aneta
  Neumann}, \bibinfo{person}{Frank Neumann}, {and} \bibinfo{person}{Tobias
  Friedrich}.} \bibinfo{year}{2019}\natexlab{}.
\newblock \showarticletitle{Pareto Optimization for Subset Selection with
  Dynamic Cost Constraints}. In \bibinfo{booktitle}{\emph{Proceedings of the
  2019 AAAI Conference on Artificial Intelligence}}
  \emph{(\bibinfo{series}{AAAI '19})}. \bibinfo{publisher}{{AAAI} Press},
  \bibinfo{pages}{2354--2361}.
\newblock


\bibitem[\protect\citeauthoryear{Storch}{Storch}{2008}]%
        {DBLP:journals/ec/Storch08}
\bibfield{author}{\bibinfo{person}{Tobias Storch}.}
  \bibinfo{year}{2008}\natexlab{}.
\newblock \showarticletitle{On the Choice of the Parent Population Size}.
\newblock \bibinfo{journal}{\emph{Evolutionary Computation}}
  \bibinfo{volume}{16}, \bibinfo{number}{4} (\bibinfo{year}{2008}),
  \bibinfo{pages}{557--578}.
\newblock


\bibitem[\protect\citeauthoryear{Ulrich and Thiele}{Ulrich and Thiele}{2011}]%
        {DBLP:conf/gecco/UlrichT11}
\bibfield{author}{\bibinfo{person}{Tamara Ulrich} {and} \bibinfo{person}{Lothar
  Thiele}.} \bibinfo{year}{2011}\natexlab{}.
\newblock \showarticletitle{Maximizing population diversity in single-objective
  optimization}. In \bibinfo{booktitle}{\emph{Proceedings of the 2011 Genetic
  and Evolutionary Computation Conference}} \emph{(\bibinfo{series}{GECCO
  '11})}. \bibinfo{publisher}{ACM}, \bibinfo{pages}{641--648}.
\newblock


\end{thebibliography}

\end{document}